\documentclass{article}

    \usepackage[preprint]{neurips_2025}

\usepackage[utf8]{inputenc} 
\usepackage[T1]{fontenc}    
\usepackage{hyperref}       
\usepackage{url}            
\usepackage{booktabs}       
\usepackage{amsfonts}       
\usepackage{nicefrac}       
\usepackage{microtype}      
\usepackage{xcolor}         
\usepackage{tcolorbox}

\usepackage{amsmath}
\usepackage{amssymb}
\usepackage{mathtools}
\usepackage{amsthm}
\usepackage{bm}
\theoremstyle{plain}
\newtheorem{theorem}{Theorem}[section]

\newtheorem{lemma}[theorem]{Lemma}

\theoremstyle{definition}

\theoremstyle{remark}

\usepackage{multirow}

\title{On the Role of Hidden States of Modern Hopfield Network in Transformer}

\author{%
Tsubasa Masumura$^{*}$ \quad
Masato Taki$^{*}$ \\
Graduate School of Artificial Intelligence and Science, Rikkyo University, Japan\\
\texttt{\{t\_masumura,taki\_m\}@rikkyo.ac.jp}
}

\begin{document}

\maketitle
\def\thefootnote{*}\footnotetext{These authors contributed equally to this work}
\def\thefootnote{\arabic{footnote}}

\begin{abstract}
Associative memory models based on Hopfield networks and self-attention based on key-value mechanisms have been popular approaches in the study of memory mechanisms in deep learning. It has been pointed out that the state update rule of the modern Hopfield network (MHN) in the adiabatic approximation is in agreement with the self-attention layer of Transformer. In this paper, we go beyond this approximation and investigate the relationship between MHN and self-attention. Our results show that the correspondence between Hopfield networks and Transformers can be established in a more generalized form by adding a new variable, the hidden state derived from the MHN, to self-attention. This new attention mechanism, modern Hopfield attention (MHA), allows the inheritance of attention scores from the input layer of the Transformer to the output layer, which greatly improves the nature of attention weights. In particular, we show both theoretically and empirically that MHA hidden states significantly improve serious problem of deep Transformers known as rank collapse and token uniformity. We also confirm that MHA can systematically improve accuracy without adding training parameters to the Vision Transformer or GPT. Our results provide a new case in which Hopfield networks can be a useful perspective for improving the Transformer architecture.
\end{abstract}

\section{Introduction}

\label{intro}

The relationship between associative memory in Hopfield networks \cite{hopfield1982neural,amari1972learning}, which has attracted interest from neuroscientists, and Transformers \cite{vaswani2017attention} based on key-value memory that have been studied in machine learning has attracted interest from the research community \cite{tyulmankov2021biological,
salvatori2021associative,
bietti2023birth,
whittingtonrelating,
hoover2023memory}.
One of the most interesting results is the finding in \cite{ramsauerhopfield, krotov2021large} that translating modern Hopfield networks into neural networks yields the Transformer architecture that has been very successful in natural language processing \cite{radford2018improving, devlin2019bert} and computer vision \cite{dosovitskiyimage}. What, then, do more general modern Hopfield networks imply for deep learning? This paper gives a concrete answer to this question.

Hopfield networks \cite{hopfield1982neural,amari1972learning} are a class of models for associative memory.
Despite these interesting properties, classical Hopfield networks have the limitation of small storage capacity.
Recently, \cite{krotov2016dense} proposed Dense Associative Memory which can achieve storage capacity that scales exponentially or power-wise with respect to the number of neurons by introducing high nonlinearity \cite{demircigil2017model}.
These models with large storage capacity are collectively referred to as modern Hopfield networks \cite{ramsauerhopfield, krotov2021large}.

\begin{figure}[h]
\centering
    \includegraphics[width=0.95\columnwidth]{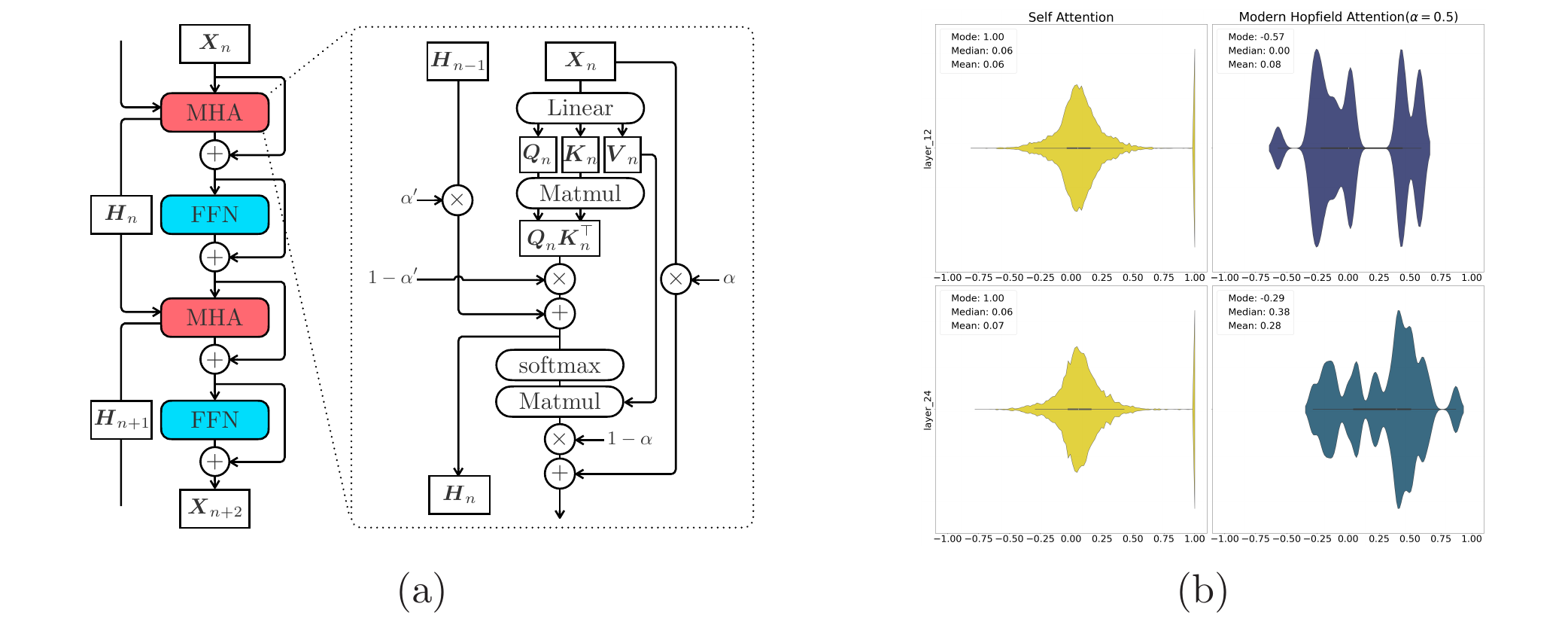}
    \caption{(a) The left figure shows the layer structure of Transformer architecture using modern Hopfield attention (MHA). As the hidden state $\boldsymbol{H}_n$ propagates through each attention layer, information from the upper layer's attention scores is reused in the lower layers. Attention score $\boldsymbol{Q}_n\boldsymbol{K}_n^\top$ is accumulated in the hidden state of each layer,  and this value is used for attention calculation. (b) A visualization of the token uniformity in layers 12 and 24 of GPT-2 (Medium) trained on the Wikitext103 dataset, showing a violin plot of the cosine similarity between the tokens. For GPT-2 in the left column, there is a strong peak at similarity 1, and both layers have a mode of 1. On the other hand, in the case of GPT-2 with MHA in the right column, the cosine similarity is kept low and the uniformity of the tokens is dramatically improved.}
    \label{fig:1}
\end{figure}

Recent advances in Transformer architecture, including its application to language models, have led to significant advances in the study of self-attention mechanisms. These advances have also shed new light on Hopfield networks. In \cite{ramsauerhopfield}, it was shown that the state update rules of modern continuous Hopfield networks (MCHNs) have a mathematical structure exactly equivalent to that of the self-attention mechanism.
Furthermore, \cite{krotov2021large} developed this relationship theoretically, pointing out that the self-attention layer of Transformer coincides with the adiabatic limit of generic modern Hopfield networks. Thus, it is expected that there is a deep relationship between the Transformer's architectural design and associative memory. Still the connection, however, lack fundamental understanding.


Therefore, we consider the question: is it possible to interpret modern Hopfield networks without adiabatic approximation in terms of Transformers?
The adiabatic limit approximation removes the hidden state dynamics from MCHN.
In this paper, we show that maintaining this dynamics introduces a hidden state on the Transformer side, thereby creating a mechanism for the propagation of attention score information from the upper to the lower layers.

By implementing this new attention mechanism into Transformer architectures, this paper introduces a new type of self-attention layer called modern Hopfield attention ({\bf{MHA}}), as shown in Figure \ref{fig:1}(a).
MHA does not require additional parameters,
and the increase in computational complexity is very small. Nevertheless, simply using MHA instead of the usual attention layer, performance gains can be obtained in various natural language processing and image recognition tasks.

Furthermore, we found that MHA effectively solves the problem known as rank collapse, or token uniformity, where Transformer's tokens lose diversity as Figure \ref{fig:1}(b). These results indicate that ideas derived from the Hopfield network may provide a new perspective for Transformer research.

In summary, our contributions are as follows:
\begin{itemize}

\item By investigating the relationship between MCHN and Transformer beyond the adiabatic approximation, we showed that this correspondence can be further generalized. 
Based on the correspondence with MCHN, we proposed a new type of attention mechanism with hidden state, MHA.
The MHA-based Transformer improves the nature of attention weights by sharing attention score information across layers.

\item By training Transformers with MHA, we experimentally showed that the MHA mechanism also contributes to the performance of the Transformers. In particular, we investigated image recognition with Vision Transformer and text generation tasks with GPT-2, and confirmed that MHA actually improves their performance. This method does not generate any additional parameters, and thus can lead to performance gains with only a small increase in computational complexity.

\item Theoretically and experimentally, we showed that the reason why MHA works so well as an alternative to self-attention is related to  rank collapse. The hidden state of MHA is likely to enhance its performance by cleverly improving Transformer's rank collapse as Figure \ref{fig:1}(b).
Our results suggest that the Hopfield Networks can provide guidance for improving the Transformer architectures.

\end{itemize}

\section{Related Works}
\label{related}

The relationship between the modern Hopfield network and the Transformer was investigated in \cite{ramsauerhopfield, krotov2021large}, and various improvements and extensions have been made to the modern Hopfield network
\cite{millidge2022universal,
zhang2022out,
iatropoulos2022kernel,
hu2023sparse,
saha2023end,
burnssimplicial,
hoover2023memory,
hu2024computational,
hu2024outlier,
wu2024uniform,
hofmannenergy,
he2024camelot}.

In \cite{ramsauerhopfield}, the authors demonstrate the fast convergence of the Modern Hopfield Network (MHN) and justify its use as a conventional module in Transformer-related architectures. On the other hand, in this study's MHA, we propose a dynamic structure that maintains and updates hidden states across layers.
More specifically, 
our MHA naturally incorporates Hopfield recursion into the Transformer layer structure, as “state accumulation and updating” are performed in each layer.

Research on improving the design of Transformer architecture using this relationship has also been conducted in \cite{hoover2023energy}. Unlike these studies, this paper focuses on the effect of keeping hidden state dynamics of MCHN.

In this paper, we saw that hidden states lead to reuse of attention scores across layers.
Attention score reuse has been studied \cite{he2021realformer, diko2024revit} from technical perspective, including improvements to the Pre-LN Transformer.
These studies are focus only on encoder architecture and do not consider the special combination of moving average with $\alpha'$ and skip connection modification with $\alpha$ as in MHA. 
On the other hand, this paper showed that the more extended attention mechanisms in MHA can be understood in terms of modern Hopfield networks, and examines its effects including the decoder Transformer. Furthermore, the essential role of MHA is clarified theoretically and experimentally in terms of rank collapse.

\section{Method: Transformers from Hopfield Network}
\label{method}

In this chapter, we review the methods \cite{ramsauerhopfield, krotov2021large} used to derive Transformer from MCHN and give a careful treatment of discretization, which has been ignored in previous studies. As a result, we show that
hidden state of MCHN leads to significant changes in the mechanisms of self-attention.

\subsection{Self-Attention Mechanism}

Let $T$ be the number of input tokens with dimension $d$.
$\boldsymbol{X}_{n}\in\mathbb{R}^{T\times d}$ is the feature obtained by concatenating the input token vectors $\boldsymbol{x}_{n}$ of the $n$-th attention layer.
The attention weight of Transformer is given by the row-wise softmax value of the attention score, which is given by the inner product of the query and the key, and the dot-product self-attention is calculated by weighting and adding the value vectors together 
as
$
\boldsymbol{X}_{n+1}=\textrm{softmax}\left(
\boldsymbol{Q}_n
\boldsymbol{K}_n^\top
\right)\boldsymbol{V}_n$,
where the query, key, value are given by linear projections of the input $\boldsymbol{X}_{n}$ as $\boldsymbol{Q}_{n}=\boldsymbol{X}_{n}\boldsymbol{W}_{Q}$,
$\boldsymbol{K}_{n}=\boldsymbol{X}_{n}\boldsymbol{W}_{K}$ and
$\boldsymbol{V}_{n}=\boldsymbol{X}_{n}\boldsymbol{W}_{V}$
\cite{vaswani2017attention}.
Each token vector is a slice $\boldsymbol{x}_{n}=(\boldsymbol{X}_{n})_{t,:}$ of the feature tensor.
Then the formula for attention mechanism for each token is $\boldsymbol{x}_{n+1}=\textrm{softmax}\left(
\boldsymbol{q}_n
\boldsymbol{K}_n^\top
\right)\boldsymbol{V}_n$.

\subsection{Modern Continuous Hopfield Network and its Discretization}

MCHN is a network model with bipartite graph connectivity connecting two dynamic variables $\boldsymbol{x}$ and $\boldsymbol{h}$. The connections are given by the network's weights $\boldsymbol{W}$, in which the memories to be associated are stored. $\boldsymbol{x}$ is called the visible state or feature neuron, and $\boldsymbol{h}$ is called the hidden state or memory neuron.
In the context of the associative memory model, given a collapsed $\boldsymbol{x}$ as an initial configuration, the complete $\boldsymbol{x}$ is reproduced by association through the time evolution of the state. The time evolution of MCHN is given by the following update rule
\cite{krotov2021large}:
\begin{align}
\tau_v\frac{d\boldsymbol{x}}{dt}
=\boldsymbol{f}\left(\boldsymbol{h}\right)\boldsymbol{W}_1^\top-\boldsymbol{x},\quad
\tau_h\frac{d\boldsymbol{h}}{dt}
=\boldsymbol{g}\left(\boldsymbol{x}\right)\boldsymbol{W}_2-\boldsymbol{h},
\end{align}
where $\tau_{v,h}$ are the time constants of the dynamic system\footnote{In the following discussion, we do not assume the tying of $\boldsymbol{W}_1$ and $\boldsymbol{W}_2$. This breaking of the symmetry of the memory matrix violates the assumption of monotonically decreasing energy function in the mathematical discussion of \cite{krotov2021large}. Interpreting the energy function of MHA in asymmetric settings is a very interesting theoretical challenge for future research.}.
The activation functions $\boldsymbol{f}(\cdot)$ and $\boldsymbol{g}(\cdot)$ are given by the Lagrangian functions $L_{h,v}$ for $\boldsymbol{h}$ and $\boldsymbol{x}$
\begin{align}
\boldsymbol{f}\left(\boldsymbol{h}\right)=\frac{\partial L_h}{\partial \boldsymbol{h}},\quad
\boldsymbol{g}\left(\boldsymbol{x}\right)=\frac{\partial L_v}{\partial \boldsymbol{x}}.
\end{align}
In this paper, the vectors
$\boldsymbol{h}=\begin{pmatrix}h_a\end{pmatrix},\,
\boldsymbol{f}\left(\boldsymbol{h}\right)=\begin{pmatrix}f_a\left(\boldsymbol{h}\right)\end{pmatrix},\,
\boldsymbol{x}=\begin{pmatrix}x_i\end{pmatrix}$ and $
\boldsymbol{g}\left(\boldsymbol{x}\right)=\begin{pmatrix}g_i\left(\boldsymbol{x}\right)\end{pmatrix}$
are all row vectors.
In order to see the correspondence with Transformer below, let us derive discrete time counterpart of MCHN. 
We then discretize this update rule with a finite difference $\Delta t=t_{n+1}-t_n$ as follows
\begin{align}
\frac{\tau_v}{\Delta t}\left(\boldsymbol{x}_{n+1}-\boldsymbol{x}_{n}\right)
=\boldsymbol{f}\left(\boldsymbol{h}_n\right)\boldsymbol{W}_1^\top-\boldsymbol{x}_n,\quad
\frac{\tau_h}{\Delta t}\left(\boldsymbol{h}_{n+1}-\boldsymbol{h}_{n}\right)
=\boldsymbol{g}\left(\boldsymbol{x}_n\right)\boldsymbol{W}_2-\boldsymbol{h}_n,
\end{align}
where
$\boldsymbol{x}_n=\boldsymbol{x}(t_{n})$ and $\boldsymbol{h}_n=\boldsymbol{h}(t_{n})$.
Introducing the ratio between the discretization step width and the time constant as
$\frac{\Delta t}{\tau_v}=1-\alpha$
and
$\frac{\Delta t}{\tau_h}=1-\alpha'$,
we
obtain
\begin{align}
\boldsymbol{x}_{n+1}=\alpha\boldsymbol{x}_{n}+(1-\alpha)\boldsymbol{f}\left(\boldsymbol{h}_n\right)\boldsymbol{W}_1^\top,\quad
\boldsymbol{h}_{n+1}=\alpha'\boldsymbol{h}_{n}+(1-\alpha')\boldsymbol{g}\left(\boldsymbol{x}_n\right)\boldsymbol{W}_2.
\end{align}
In past studies, the effect of the discretization step $\alpha$ was ignored as negligible, but the precise derivation here leads to an interesting modification of Transformer. In this paper, we give empirical and theoretical results in which $\alpha$ and $\alpha'$'s effect is extremely important.
 
\subsection{Adiabatic Limit and Self-Attention}

Specific MCHN model is determined by explicitly selecting the Lagrangians.
Model B of \cite{krotov2021large} is given by the following choice of Lagrangians
\begin{align}
L_h=\log\left(\sum_a e^{h_a} \right),\quad
L_v=\frac{1}{2}\|\boldsymbol{x}\|_2^2.
\end{align}
These Lagrangians give the activation functions
\begin{align}
\label{model-b}
f_a=\textrm{softmax}\left(h_a\right),\quad
g_i=x_i.
\end{align}
The adiabatic limit in \cite{krotov2021large}
$\tau_h\approx0$
implies
$\boldsymbol{h}_n=\boldsymbol{x}_n\boldsymbol{W}_2$.
The update rule for  (\ref{model-b}) is then given by
\begin{align}
\label{SAfromMHN}
&\boldsymbol{x}_{n+1}=\alpha\boldsymbol{x}_n
+
(1-\alpha)
\textrm{softmax}\left(\boldsymbol{x}_n\boldsymbol{W}_2\right)
\boldsymbol{W}_1^\top.
\end{align}
Translating (\ref{SAfromMHN}) by the rule 
$\boldsymbol{q}_{n}=\boldsymbol{x}_{n}\boldsymbol{W}_Q,\, \boldsymbol{W}_1^\top=\boldsymbol{X}_n\boldsymbol{W}_V=\boldsymbol{V},\, \boldsymbol{W}_2^\top= \boldsymbol{X}_n \boldsymbol{W}_K \boldsymbol{W}_Q^\top = \boldsymbol{K} \boldsymbol{W}_Q^\top$ according to \cite{ramsauerhopfield, krotov2021large}, we obtain $\boldsymbol{x}_{n+1}=
\alpha\boldsymbol{x}_{n}
+(1-\alpha)
\textrm{softmax}\left(\boldsymbol{q}_n\boldsymbol{K}^\top\right)
\boldsymbol{V}
$, where
$\boldsymbol{X}_n$ is the concatenated tensor of the embedding vectors of all tokens 
$\boldsymbol{x}_{n1},\cdots,\boldsymbol{x}_{nT}$.
When $\alpha=0$, i.e., $\Delta t=\tau_v$, this update rule is exactly a usual self-attention mechanism in \cite{vaswani2017attention}. In the following, we consider general $\alpha$ and $\alpha'$ to investigate the effect of the hidden state dynamics, which are ignored in the  adiabatic limit above.

\subsection{Hidden State Dynamics and Modern Attention Attention}

If the adiabatic limit is not taken and a finite $\frac{\Delta t}{\tau_h}$ is kept, the dynamics of the hidden state is
\begin{align}
\frac{\tau_h}{\Delta t}\left(\boldsymbol{h}_{n+1}-\boldsymbol{h}_{n}\right)
=\boldsymbol{g}\left(\boldsymbol{x}_{n+1}\right)\boldsymbol{W}_2-\boldsymbol{h}_{n+1}.
\end{align}
In the following, we use the new parameterization $\frac{\Delta t}{\tau_h}=\frac{1-\alpha'}{\alpha'}$ to obtain a simple formula.
The dynamics of Model B is then
\begin{align}
\boldsymbol{x}_{n+1}=\alpha\boldsymbol{x}_{n}+(1-\alpha)\textrm{softmax}\left(\boldsymbol{h}_n\right)\boldsymbol{W}_1^\top,\quad
\boldsymbol{h}_{n+1}=\alpha'\boldsymbol{h}_{n}+(1-\alpha')\boldsymbol{x}_{n+1}\boldsymbol{W}_2.
\end{align}
Using the same translation rules as before, we get the following novel modification of attention layer
\begin{align}
\label{hsd1}
\boldsymbol{x}_{n+1}=\alpha\boldsymbol{x}_{n}+(1-\alpha)\textrm{softmax}\left(\boldsymbol{h}_n\right)\boldsymbol{V}_n,\quad
\boldsymbol{h}_{n}=\alpha'\boldsymbol{h}_{n-1}+(1-\alpha')\boldsymbol{q}_{n}\boldsymbol{K}_{n}^\top.
\end{align}
Thus, if the dynamics of the hidden state in the MHN is maintained and mapped to the self-attention layer, a new variable $\boldsymbol{h}$, determined by the value of the attention scores, is added to the self-attention layer. This variable continues to accumulate the value of the attention score in each layer in the form of an exponential moving average across layers. Through this variable, the attention weights of each layer of the Transformer will have a coordinated behavior. In the following, we will investigate the effect of adding this hidden state on the attention layer from the Transformer's perspective.
In this paper, this extended attention mechanism with hidden states will be referred to as {\bf{M}}odern {\bf{H}}opfield {\bf{A}}ttention ({\bf{MHA}}).
Compared to the cost $O(dT^2)$ of computing the dot product of self-attention, the computational complexity added by updating the hidden state is about $O(T^2)$. For Transformer that uses more than several hundred dimensions of $d$, this is a small increase in computational complexity.

\section{Empirical Results}
\label{exp}

In this chapter, we experimentally investigate how the performance of the model changes when MHA is actually used in place of Transformer's self-attention module.
We take the Vision Transformer (ViT) as a representative example of an encoder Transformer model and the GPT-2\cite{radford2019language} architecture as a representative example of a decoder Transformer model, and confirm that MHA does indeed lead to systematic performance improvements in several experiments.

\subsection{Architecture with MHA}

In the following, we will focus on the simplest case $\alpha=\alpha'$. It is straightforward to choose both parameter independently, but consider only this case to reduce the hyperparameters. 
By using our update rule (\ref{hsd1})
instead of the attention layer, a new tensor called the hidden state $\boldsymbol{H}_{\ell}$ propagates across the layers. This tensor accumulates the attention score $\boldsymbol{Q}_{\ell}\boldsymbol{K}_{\ell}^\top$ in each layer in the form of an exponential moving average. It is not the original attention score that gives the attention weight, but the softmax of the hidden state.
At the same time, a skip connection with the weight $(1-\alpha)$ linked to the coefficient $\alpha$ of the exponential moving average of the hidden state is added according to equation (\ref{hsd1}), and the balance between the two effects, controlled by $\alpha$, is considered to determine the behavior of the MHA.
The detailed structure of the architecture corresponding to (\ref{hsd1}) is illustrated in Figure \ref{fig:1}(a).

In the following experiments, we will employ scaled dot-product attention according to the usual Transformer design and introduce the coefficient $\frac{1}{\sqrt{d_k}}$ in the argument of the softmax function.

\subsection{Text Generation: GPT-2}
\label{resources}
To determine the impact of MHA on Transformer performance, we first trained GPT-2 Small(124M) and Medium(350M) \cite{radford2019language} on text generation task and tested their performance.
The dataset used was WikiText103 \cite{merity2017pointer}.
The following experiments in this paper were conducted using up to eight A100 GPUs.
The detailed training settings are described in the supplemental material.

To fairly compare the effectiveness of MHA, we trained the GPT-2 architecture and an architecture in which the self-attention layers of GPT-2 are replaced by MHA in the same setting from scratch and compared their perplexity.
Table \ref{pplgpt2} shows the results.
The interest of this paper is not to create a SOTA model with detailed hyperparameter tuning, etc., but to see the robustness of the MHA effect, so $\alpha$ was simply set to $0.5$ based on rough hyperparameter search.

As Table \ref{pplgpt2}  shows, there is a clear improvement in perplexity in both the Small and Medium MHA models. 
Hopfield networks have often been experimented with in comparison to encoder Transformers \cite{ramsauerhopfield}, but our result shows that such comparisons is also useful for decoders.

\begin{table}[h]
\caption{Comparison of the perplexity of GPT-2 and its MHA counterpart trained on the WikiText103 dataset for two cases: GPT-2 Small with 124M parameters and GPT-2 Medium with 350M parameters. In both cases, the introduction of MHA improved the perplexity.}
\label{pplgpt2}
\centering
\begin{tabular}{cc|cc}

\hline
    \multicolumn{2}{c|}{Small(124M)}&\multicolumn{2}{c}{Medium(350M)}   \\
\hline
    $\textrm{self-attention}$&$\textrm{MHA}(\alpha=0.5)$&$\textrm{self-attention}$&$\textrm{MHA}(\alpha=0.5)$ \\
\hline\hline
    $22.87$&$\bm{20.70}$&$20.85$&$\bm{19.61}$ \\
\hline

\end{tabular}
\end{table}

\subsection{Text Generation: LLaMA Architecture}

To evaluate the effectiveness of Modern Hopfield Attention in more practical text generation architectures, we conducted additional experiments on LLaMA, in addition to GPT-2, using the miniLLaMA implementation. Furthermore, besides WikiText-103, we individually examined cases where CNN DailyMail \cite{hermann2015teaching} and BookCorpus \cite{zhu2015aligning} were used as training datasets. The results are summarized in Table \ref{tab:llama}. Even in practical architectures such as LLaMA, whose refined design aims to enhance performance, MHA was found to exert a consistent improvement in perplexity, demonstrating its systematic effectiveness beyond simpler baseline models.
\begin{table}[ht]
    \centering
    \begin{tabular}{c|cc}
        \hline
        dataset           & self-attention       & MHA      \\
        \hline\hline
        WikiText-103          & $14.49$      & $\bm{14.29}$  \\
        \hline
        DailyMail          & $19.36$      & $\bm{18.97}$  \\
        \hline
        BoocCorpus          & $23.76$      & $\bm{23.50}$  \\
    \end{tabular}
    \caption{Comparison of the perplexity of LLaMA and its MHA counterpart ($\alpha=0.5$) trained on various datasets. In all cases, the introduction of MHA led to improved perplexity.}
    \label{tab:llama}
\end{table}

\subsection{Image Recognition: ViT}

Next, the Vision Transformer (ViT) was employed as the Transformer decoder model, and again to fairly compare the effect of MHA, two architectures, the ViT architecture and the architecture in which the self-attention layers of ViT are replaced by MHA, were trained in the same configuration. 
We trained these models in image recognition tasks.

The model used in this study is ViT \cite{dosovitskiyimage}, and the data sets used are CIFAR10/CIFAR100 \cite{Krizhevsky09learningmultiple} and ImageNet-1k \cite{5206848}.
The detailed training setup is shown in the supplemental material.

\begin{table}[ht]
    \centering
    \begin{tabular}{c|c|cc}
        \hline
        model size & model type           & CIFAR10       & CIFAR100      \\
        \hline\hline
        ViT-Tiny(5.5M)       & self-attention       & $93.265$      & $\bm{73.080}$  \\
        \quad      & MHA($\alpha=0.5$) & $93.015$      & $72.030$       \\
        \quad      & MHA($\alpha=0.7$) & $\bm{93.775}$ & $72.570$      \\
        \hline
        ViT-Small(22M)      & self-attention       & $\bm{95.450}$ & $74.485$                       \\
        \quad      & MHA($\alpha=0.5$) & $95.335$      & $75.420$                          \\
        \quad      & MHA($\alpha=0.7$) & $95.440$      & $\bm{75.590}$                     \\
        \hline
        ViT-Base(86M)      & self-attention       & $96.190$      & $75.360$                 \\
        \quad      & MHA($\alpha=0.5$) & $96.175$      & $\bm{76.215}$          \\
        \quad      & MHA($\alpha=0.7$) & $\bm{96.490}$ & $75.590$            \\
        \hline
        ViT-Large(303M)     & self-attention       & $96.310$      & $72.910$                      \\
        \quad      & MHA($\alpha=0.5$) & $96.500$      & $\bm{75.775}$                     \\
        \quad      & MHA($\alpha=0.7$) & $\bm{96.690}$ & $75.365$                          \\
    \end{tabular}
    \caption{Experimental results are shown for ViTs and their MHA counterparts. For simple tasks such as CIFAR10, performance is close to saturation and there is no clear effect of MHA. On the other hand, for CIFAR100, the performance improvement due to MHA is clear for the larger model. This is a common property of $\alpha=0.5$ and $\alpha=0.7$.}
    \label{tab:result_cifar10}
\end{table}

\subsubsection{CIFAR10/100}

First, as a simple case, we review the results for CIFAR10 in the left column of the Table \ref{tab:result_cifar10}; for CIFAR10, the effect of MHA is not clearly visible, partly because the performance is basically close to saturation due to the ease of the task. However, it is interesting to note that the effect of MHA is starting to appear in the Base and Large models, which have a high learning capacity. In any case, CIFAR10 is not a sufficient task for the purpose of observing changes in Transformer performance with scratch training.

So let's look at the results for CIFAR100, where the task is more difficult: as shown in Table \ref{tab:result_cifar10}, the larger the model, the larger and clearer the improvement compared to the baseline ViT. Interestingly, in both cases of the two $\alpha$ choices shown here, the performance improvement relative to ViT can be seen when the model is larger than the Small model.

\subsubsection{ImageNet-1k}

In the experiments on ImageNet-1k, due to computational resource constraints, we adopt ViT-B (86M) as a model of good enough size to obtain nontrivial training results. 
The results of 300-epoch training of ViT-B and its MHA counterpart from scratch with ImageNet-1k are shown in Table \ref{tab:result_in1k}.
Following the standard training setup, AdamW\cite{loshchilovdecoupled} was used for optimizer and cosine decay for learning rate scheduling. Random erasing\cite{zhong2020random}, mixup\cite{zhang2018mixup}, cutmix\cite{Yun_2019_ICCV}, and RandomAugment\cite{cubuk2020randaugment} were used for augmentation. For details, please refer to the supplemental material.

\begin{table}[h]
\centering
\caption{Classification validation accuracies for ViT-B($86\textrm{M}$) and its MHA counterpart.}
\begin{tabular}{c|c|c|c}
\hline
Data set & self-attention & MHA($\alpha=0.5$) & MHA($\alpha=0.7$)\\
\hline\hline
ImageNet-1k   &$76.074$& ${76.434}$ & $\bm{77.058}$
\label{tab:result_in1k}
\end{tabular}
\end{table}

As shown in Table \ref{tab:result_in1k}, the performance improvement in ViT-B was also observed in ImageNet-1k.
Although the performance improvement is less than 1\%,
 this performance difference is considered a non-trivial result compared to the examples in previous studies on ViT improvement. 
As in previous experiments, this increase in performance is produced by adding only a small amount of computation without adding any training parameters.
This is an interesting result, which suggests that hidden states may help improve attention mechanisms.
In the next chapter, we will investigate both theoretically and experimentally regarding how hidden states produce these performance gains.

\subsubsection{Downstream Tasks}
To evaluate MHA's effectiveness across diverse tasks, we measured the transfer performance of a pre-trained ImageNet model using linear probing. Using a pre-trained ViT and its MHA counterpart as backbones, we conducted transfer learning experiments on four commonly used downstream datasets (Oxford Flowers 102 \cite{nilsback2008automated}, Food-101 \cite{bossard2014food}, Stanford Dogs \cite{khosla2011novel}, and Stanford Cars \cite{krause20133d}). Results in Table \ref{tab:linear_probe} show that the MHA variants achieve consistently good transfer performance.
\begin{table}[ht]
    \centering
    \begin{tabular}{c|cc}
        \hline
        dataset           & self-attention       & MHA      \\
        \hline\hline
        Flower102          & $81.15$      & $\bm{93.85}$  \\
        \hline
        Food101          & $74.51$      & $\bm{87.99}$  \\
        \hline
        Stanford dogs          & $\bm{95.00}$      & ${83.64}$  \\
        \hline
        Stanford cars          & $51.54 $      & $\bm{87.54}$  \\
    \end{tabular}
    \caption{Comparison of the transfer accuracy of ImageNet-1K pre-trained ViT and its MHA counterpart ($\alpha=0.7$) on various downstream datasets.}
    \label{tab:linear_probe}
\end{table}

\subsection{Effect of Combining $\alpha$ and $\alpha'$}

Our update rule (\ref{hsd1}) has two hyperparameters $\alpha$ and $\alpha'$, but for simplicity, we have so far restricted our discussion to the case where both values are equal $\alpha=\alpha'$. However, as shown in Figure \ref{fig:1}(a), these two quantities essentially work differently. 
$\alpha$ is a quantity that balances
the value after attention computation
and the strength of the skip connections in the attention module.
On the other hand, $\alpha'$ is the coefficient of the exponential moving average in accumulating the attention scores to hidden states.

A nontrivial result (\ref{hsd1}) derived from MCHN is that these two independent effects are simultaneously added to the Transformer. To see whether these two are really both necessary, or whether they work in concert, let us try an experiment in which $\alpha$ and $\alpha'$ are varied independently.

\begin{table}[h]
\centering
\caption{Performance change of ViT-T when two hyperparameters are changed independently}
\resizebox{0.95\textwidth}{!}{%
\begin{tabular}{c|ccccccccccc}
\multicolumn{12}{l}{CAIFAR100  MHA($\alpha=0.5)$} \\
\hline
$\alpha'$ & $0.0$ & $0.1$ & $0.2$ & $0.3$ & $0.4$ & $0.5$ & $0.6$ & $0.7$ & $0.8$ & $0.9$ & $1.0$ \\
\hline
score & $71.16$ & $71.13$ & $72.29$ & $70.77$ & $72.12$ & $72.13$ & $72.06$ & $72.02$ & $71.98$ & $70.72$ & $66.10$ \\
\hline
\hline

\multicolumn{12}{l}{CAIFAR100  MHA($\alpha'=0.5)$} \\
\hline
$\alpha$ & $0.0$ & $0.1$ & $0.2$ & $0.3$ & $0.4$ & $0.5$ & $0.6$ & $0.7$ & $0.8$ & $0.9$ & $1.0$ \\
\hline
score & $69.89$ & $71.20$ & $71.26$ & $71.64$ & $72.02$ & $72.13$ & $72.66$ & $70.52$ & $70.46$ & $67.70$ & $\phantom{0}1.00$ \\
\hline
\label{tab:alphachange}
\end{tabular}
}
\end{table}

Table \ref{tab:alphachange} shows the change in performance when one of the hyperparameters is fixed at 0.5 and the value of the other is varied. When only $\alpha$ is moved from the original $\alpha=0.5=\alpha'$ to $\alpha=1$, the performance drops to chance-level accuracy. This is evident from the fact that all values except for the skip connection are set to $0$. On the other hand, when $\alpha$ is set to $0$, the performance degrades to $69.89$. Thus, it can be seen that further performance improvement is realized by adding not only $\alpha'$ but also $\alpha$. Similarly, when $\alpha$ is fixed,
setting $\alpha'$ to $0$ also results in poorer performance.

\section{How MHA improves Transformers}
\label{theory}

\subsection{Problem and Improvement of Transformer Layers}
Next, let us examine why MHA leads to performance gains in various tasks. It is known that as the depth of Transformer increases, training becomes more difficult and performance tends to saturate rapidly. The phenomenon of rank collapse has been discussed as one cause of this problem. It is possible that our model mitigates the problem without explicit regularization or other means. Therefore, we provide below some theoretical and empirical results that support this hypothesis.

\subsubsection{Rank Collapse}

It has been observed that as the depth of the Vision Transformer increases, the patch tokens become extremely similar and rapidly lose diversity \cite{tang2021augmented}\cite{zhou2021deepvit}. This phenomenon is now understood as token uniformity, rank collapse, or oversmoothing \cite{dong2021attention, tang2021augmented, zhou2021deepvit, yan2022addressing, noci2022signal, shi2022revisiting, he2023deep, guo2023contranorm, park2022vision, wang2022anti, bai2022improving, dovonon2024setting}\footnote{The rank collapse in \cite{dong2021attention}  refers to the phenomenon where the tokens corresponding to each row of a feature become perfectly proportional vectors. This means perfect token uniformity. On the other hand, the phenomenon observed in actual Transformers is that many, if not all, tokens are perfectly aligned, forming a group of tokens with a mutual cosine similarity of $1$.}. Various innovations have been proposed to reduce this problematic phenomenon in order to improve Transformer performance \cite{tang2021augmented, zhou2021deepvit, bai2022improving}.

Rank collapse \cite{dong2021attention} is defined as a phenomenon in which Transformer feature rapidly collapses into a rank 1 matrix
with increasing depth.
Thus, for the feature $\boldsymbol{X}^{(L)}$ of the $L$-th layer, rank collapse is formulated as the property that the residual of deep Transformer feature rapidly converges to zero as follows 
\begin{align}
\|\textrm{Res}(\boldsymbol{X}^{(L)})\|\approx 0
\textrm{ for }L\gg 1
,
\end{align}
where $\textrm{Res}(\boldsymbol{X})$ is the residual
$\textrm{Res}(\boldsymbol{X})=
\boldsymbol{X}-\boldsymbol{1}\boldsymbol{x}^\top$
for
$\boldsymbol{x}=
\arg\min_{\boldsymbol{x}}
\|\boldsymbol{X}-\boldsymbol{1}\boldsymbol{x}^\top\|$.
This convergence means the feature is approximately a rank one matrix $\boldsymbol{X}^{(L)}\approx \boldsymbol{1}\boldsymbol{x}^\top$.

\subsection{Theoretical Implication}
For a clear theoretical analysis of the causes of rank collapse, we consider a deep network consisting of only the self-attention layers according to \cite{dong2021attention}.
It is also straightforward to extend the discussion to the actual Transformer architecture \cite{dong2021attention}.
Let us consider
a self-attention-only network consisting of $L$ layers without skip connection
\begin{align}
{\textrm{AttnNet}}(\boldsymbol{X})
=
\textrm{MHSA}\circ\cdots\circ
\textrm{MHSA}(
\boldsymbol{X}).
\end{align}
$\textrm{MHSA}(
\boldsymbol{X})$
is the multi-head self-attention module.
The number of heads and embedding dimension of each MHSA are $H$ and $d_k$.
In \cite{dong2021attention},
this attention-only network has been shown to cause very serious rank collapse:
\begin{theorem}[\cite{dong2021attention}]
\label{thm:double-exp}
The norm of the residual of attention-only network ${\textrm{AttnNet}}(\boldsymbol{X})$
decays as
\begin{align}
\label{eq:double-exp}
\|
\textrm{Res}\left(
{\textrm{AttnNet}}(\boldsymbol{X})\right)
\|_{1,\infty}
\leq
\left(rC\right)^{\frac{3^L-1}{2}}
\|
\textrm{Res}\left(
\boldsymbol{X}\right)
\|_{1,\infty}^{3^L},
\end{align}
where $r=\frac{8H}{\sqrt{d_k}}$ and $C$ is certain constant.
This suggests the double exponential decay of the rank.
\end{theorem}
The definition of the norm $\|\cdot\|_{1,\infty}$ in this paper is the  composite of operator norms
$
\|\boldsymbol{X}\|_{1,\infty}
=
\sqrt{
\|\boldsymbol{X}\|_{1}
\|\boldsymbol{X}\|_{\infty}
}$.

In \cite{dong2021attention}, it was shown that skip connection and the addition of an FFN layer are effective in reducing this serious collapse. 

Interestingly, however, even though our MHA is not specifically designed to prevent rank collapse, it is able to prevent the decay phenomenon in attention-only networks without any skip connection.
Even when removing skip connections completely by setting $\alpha=0$, a non-zero $\alpha'$ leads to the following mitigation of rank collapse in the attention-only network:
\begin{theorem}
\label{thm:no-double-exp}
By keeping non-zero $\alpha'$,
the upper-bound of inequality evaluation is improved as follows
\begin{align}
\label{eq:thm2}
\|
{\textrm{Res}}\big(
{\textrm{AttnNet}}(\boldsymbol{X})
\big)
\|_{1,\infty}
\leq
\max{}_{m=0}^{L}
\left(
{r(1-\alpha')C_1}
\right)^{\frac{3^m-1}{2}}
\left(
{r\alpha' C_2}
\right)^{3^m(L-m)}
\|
{\textrm{Res}}\big(
\boldsymbol{X}
\big)
\|_{1,\infty}^{3^m}
.
\end{align}
This suggests the avoidance of exponential decay.
\end{theorem}
\begin{proof} 
See the supplemental material for detailed proof.
The sketch of the proof is as follows: by introducing the hidden state as $\alpha'\neq0$, the decaying effect of rank by single attention layer can be evaluated as follows 
\begin{align}
\|
{\textrm{Res}}\big(
{\textrm{MHSA}}(\boldsymbol{X}
\big)
\|
\leq
\max\left(
r_1(1-\alpha')
\|
{\textrm{Res}}\big(
\boldsymbol{X}
\big)
\|^3
,\,
r_2\alpha'
\|
{\textrm{Res}}\big(
\boldsymbol{X}
\big)
\|
\right)
,
\end{align}
where $r_{1,2}=rC_{1,2}$ and the norm here is $\|\cdot\|_{1,\infty}$.
Notice that the second argument in the max function significantly reduces the third-order decaying effect in \cite{dong2021attention}. By applying this inequality repeatedly over $L$ layers, we obtain the following inequality
\begin{align}
\|
{\textrm{Res}}\big(
{\textrm{AttnNet}}(\boldsymbol{X})
\big)
\|
\leq
\max\left(
\left(r_1(1-\alpha')\right)^{\frac{3^L-1}{2}}
\|
{\textrm{Res}}\big(
\boldsymbol{X}
\big)
\|^{3^L}
,
\cdots
,
\left(r_2\alpha'\right)^L
\|
{\textrm{Res}}\big(
\boldsymbol{X}
\big)
\|
\right)
,
\end{align}
where $\cdots$ means $(r_1(1-\alpha'))^{(3^m-1)/2}(r_2\alpha')^{3^m(L-m)}(\| \textrm{Res}(\boldsymbol{X})\|)^{3^m}$ for $m=1,\cdots,L-1$.
\end{proof} 
On the right hand side of this inequality (\ref{eq:thm2}), 
 the $m=L$ term is the very term that created the double exponential decay of the original self-attention mechanism \cite{dong2021attention}, but the $m=0$ term dominates in (\ref{eq:thm2}) and relaxes the rank decay to linear decay as 
$\left(
{r\alpha' C_2}
\right)^L
\|
{\textrm{Res}}\big(
\boldsymbol{X}
\big)
\|_{1,\infty}$ since
\begin{align}
(r_1(1-\alpha'))^{(3^m-1)/2}(r_2\alpha')^{3^m(L-m)}(\| \textrm{Res}(\boldsymbol{X})\|)^{3^m}
<
(r_2\alpha')^{L}(\| \textrm{Res}(\boldsymbol{X})\|)^L.
\end{align}
Note that we assume $r_{1,2}, \| \textrm{Res}(\boldsymbol{X})\|<1$ following the logic of \cite{dong2021attention}.
This decaying factor is controlled by the hidden states of the $h$-th head of the $\ell$-th layer
$
\boldsymbol{H}_{\ell,h}=\alpha'\boldsymbol{H}_{\ell-1,h}+(1-\alpha')\boldsymbol{Q}_{\ell,h}\boldsymbol{K}_{\ell,h}^\top
$
and the weight matrix $\boldsymbol{W}^{(\ell)}_{VO,h}$ for the value and output linear projection of attention module as
$C_2=\max_\ell\max_h \|\boldsymbol{W}^{(\ell)}_{VO,h}\|_{1,\infty}\|\boldsymbol{H}_{\ell,h}\|_1$.

In \cite{dong2021attention}, such an effect was created by introducing skip connection,
but in the MHA, the hidden state contribution already produces such an effect without using skip connection.
Also, setting $\alpha'=0$ reproduces the double exponential decay results of the original attention-only network (\ref{eq:double-exp}).

\subsection{Empirical Results} 
Using the theoretical analysis setup used in previous studies, we showed that MHA can effectively prevent rank collapse in the previous section.
However, since these setups are based on several theoretical simplifications, it is unclear whether the rank collapse reduction also occurs in actual Transformers. In particular, it is not clear whether the introduction of MHA has any further effect in usual architectures with skip connection to reduce rank collapse. 
In this section, we will confirm that MHA does indeed further reduce rank collapse in a few controlled experiments.

\begin{table}[ht]
    \centering
    \caption{Changes in performance as skip-free networks based on ViT-T are deepened.}
    \begin{tabular}{c|cc|cc}
        \hline
        \multirow{2}{*}{depth} & \multicolumn{2}{c}{self-attention} & \multicolumn{2}{|c}{MHA}                      \\
                               & CIFAR10& CIFAR100& CIFAR10 & CIFAR100 \\
        \hline\hline
        1 & $55.08\phantom{\uparrow}$        & $30.90\phantom{\uparrow}$  & $65.41\phantom{\uparrow}$ & $40.08\phantom{\uparrow}$          \\
        2 & $63.72\uparrow$                    & $40.06\uparrow$              & $79.75\uparrow$           & $56.94\uparrow$  \\
        4 & $57.38\downarrow$                  & $32.25\downarrow$      & $85.74\uparrow$           & $64.39\uparrow$        \\
        8 & $48.59\downarrow$                  & $17.19\downarrow$      & $80.34\downarrow$         & $49.90\downarrow$      \\
        12  & $10.00\downarrow$                 & $\phantom{0}1.00\downarrow$   & $10.00\downarrow$         & $\phantom{0}1.00\downarrow$   \\
        \hline
    \end{tabular}
    \label{tab:ablation_tiny}
\end{table}


Since the skip-free network was shown to suffer from rank collapse as it gets deeper, let's examine the effect of MHA on the actual performance degradation with depth.
Table \ref{tab:ablation_tiny} shows the results of trained models from depths 1 to 12 for the skip-free networks and their MHA versions, and evaluating their performance.
As can be seen from the results in the left column of the table, when the depth increases beyond 4 layers, the performance drops sharply due to multilayering. On the other hand, for the models in the right column using MHA, it can be seen that the degradation of the model due to multilayering is kept at a fairly mild level.
Thus, the MHA model can effectively utilize the depth of the model than the original model.

Therefore, it is highly likely that MHA can significantly improve rank collapse, which becomes more severe as the network becomes deeper, even in real networks.

\begin{figure}
    \centering
    \includegraphics[width=0.97\linewidth]{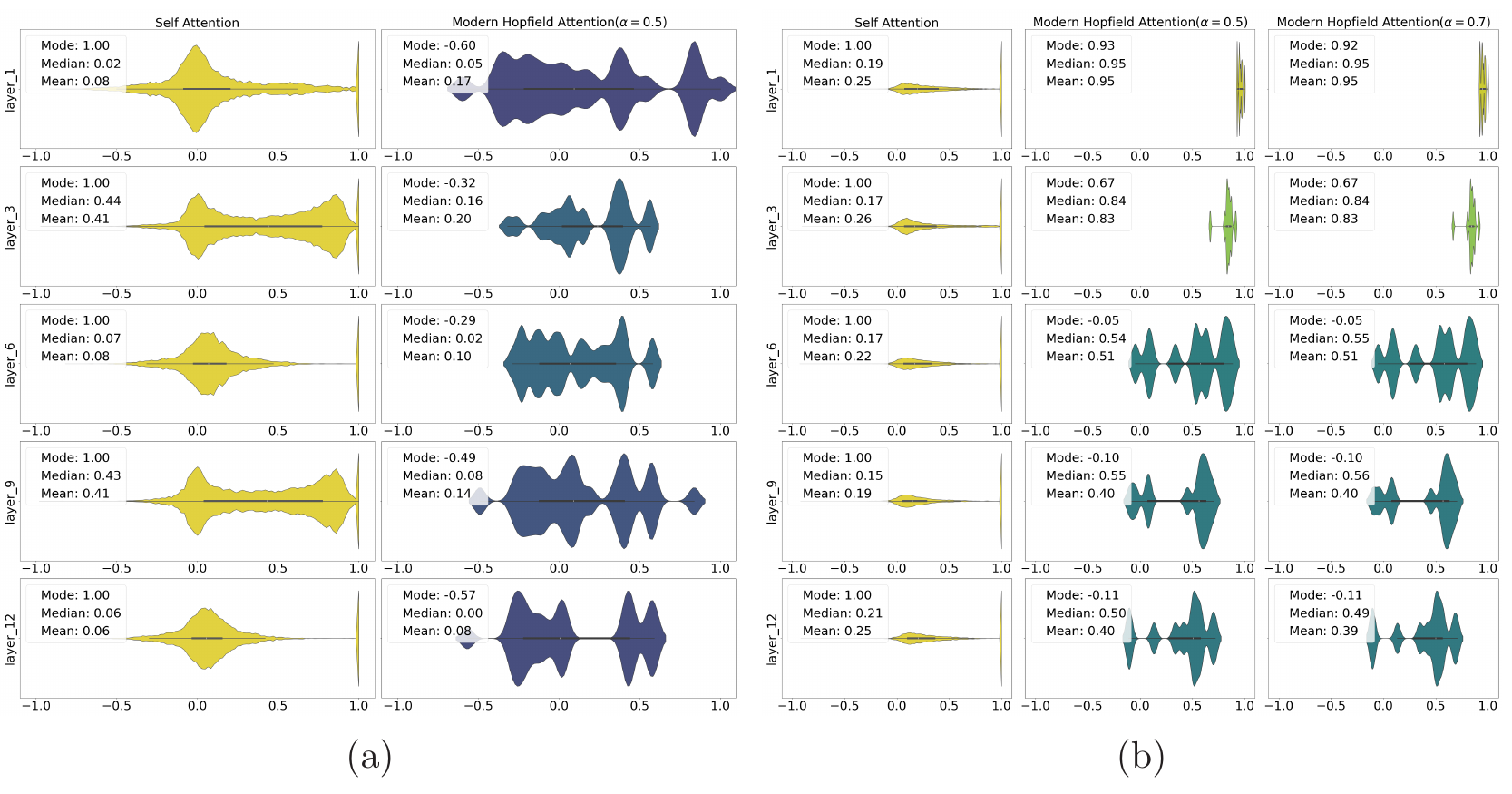}
    \caption{The violin plots of cosine similarity between tokens in several layers for (a) GPT-2 (Medium) trained on Wikitext103 and (b) ViT-B trained on CIFAR100.
    MHA layers with high average similarity of tokens exist, but tokens with a perfect similarity of 1, as in the case of self-attention, disappear, preventing their ranks from dropping.
    }
    \label{fig:enter-label}
\end{figure}

Next, let us examine cases of actual Transformer architectures with skip connections and FFN layers.
Figure \ref{fig:enter-label} shows the measured cosine similarity between tokens in several layers for GPT-2 (Medium) and ViT-B trained on CIFAR100, displayed as violin plots. See the supplementary material for more detailed plots. It is noteworthy that in the cases of normal  GPT-2 and ViT-B, the mode of similarity is $1.0$ for all layers, while the violin plot shows a sharp peak around $1.0$. This indicates that even with the addition of the skip connection and FFN layers, there is still a non-negligible token uniformity, or partial rank collapse.
On the other hand, the results for the GPT-2 and ViT models with MHA show that the peaks in the original models have disappeared and the mode values have been reduced to very small values. This indicates that MHA does indeed play a role in dramatically removing token uniformity in GPT-2 and ViT.

\section{Conclusion}
\label{discussions}

In this paper, we examine the question of whether new insights can be obtained from the modern Hopfield network for Transformer. The results showed that by introducing the hidden state of MCHN into Transformer, a new attention mechanism called MHA, which inherits attention scores from layer to layer, has been discovered and can be useful for improving ViT and GPT performance. 
MHA was also found to play a role in solving the rank collapse problem in deep Transformer. 
The MHA's mechanism to prevent the rank collapse may have contributed to Transformer's improved performance.
We hope that this research will open new possibilities for the systematic design of Transformer architectures using Hopfield networks.

\newpage

\bibliography{papers} 
\bibliographystyle{plain} 

 \appendix

\section{Acknowledgment}
MT was partially supported by JSPS KAKENHI (22H05116), JST CREST (JPMJCR22N4) and AMED under Grant Number JP25wm0625422.

\section{Limitations}

The limitation of this paper is that due to the constraints of computer resources, the experiments are limited to academic scale: At most, GPT-2 (Medium) trained on Wikitext103, ViT-L trained on CIFAR100, and ViT-B trained on ImageNet-1k. It will be an interesting future direction to see what results will be obtained with larger-scale pre-training. In addition, this paper mainly considered rank collapse and attention entropy as factors that contribute to the performance improvement of MHA, but it may be interesting to investigate whether there is any relationship with other factors. In this paper, the research was conducted based on the similarity of the mathematical structure between Hopfield networks and Transformer, but clarifying the essential deep relationship between associative memory and self-attention mechanisms is also a major future challenge.

\section{Preliminaries
}

Recall the definition of the operator $p$-norm:
\begin{align}
\|\boldsymbol{A}\|_p
=
\max_{\boldsymbol{x}\neq\boldsymbol{0}}
\frac{\|\boldsymbol{A}\boldsymbol{x}\|_p}{\|\boldsymbol{x}\|_p}.
\end{align}
The right-hand side is computed using the vector $p$-norm $\|\|_p$.

We will use some basic properties of the operator norm.
First, the $1$-norm and $\infty$-norm are given in the following simple form:
\begin{align}
&
\label{prelim1}
\|\boldsymbol{A}\|_1
=
\max_j
\sum_i
\vert A_{ij}\vert 
,\\
&
\label{prelim2}
\|\boldsymbol{A}\|_\infty
=
\max_i
\sum_j
\vert A_{ij}\vert.
\end{align}
The submultiplicativity of these two norms is
the following properties
\begin{align}
&
\|\boldsymbol{A}\boldsymbol{B}\|_1
\leq
\|\boldsymbol{A}\|_1\|\boldsymbol{B}\|_1,\\
&
\|\boldsymbol{A}\boldsymbol{B}\|_\infty
\leq
\|\boldsymbol{A}\|_\infty\|\boldsymbol{B}\|_\infty.
\end{align}
We have the same property for another choice of norm $\|\|_p$.

We also have the following H\"{o}lder's inequality for $1$- and $\infty$-norms
\begin{align}
\label{holder}
\|\boldsymbol{A}\boldsymbol{B}\|_1
\leq
\|\boldsymbol{A}\|_\infty
\|\boldsymbol{B}\|_1.
\end{align}
Proof is straightforward as
$
\|\boldsymbol{A}\boldsymbol{B}\|_1
=
\max_i\sum_{k}
\vert A_{ik}\vert 
\sum_{j}
\vert B_{kj}\vert 
\leq
\max_i\sum_{k}
\vert A_{ik}\vert 
\max_{k'}\sum_{j}\vert B_{k'j}\vert
$
$=
\|\boldsymbol{A}\|_\infty\|\boldsymbol{B}\|_1$.

Following the paper \cite{dong2021attention}, we also use the following composite of operator norms in this paper
\begin{align}
\|\boldsymbol{A}\|_{1,\infty}
=
\sqrt{
\|\boldsymbol{A}\|_{1}
\|\boldsymbol{A}\|_{\infty}
}.
\end{align}

\section{Proof of Theorem: Convergence of the Residual}

In this section, we present the proof of theorems on the phenomenon of rank collapse.
In the following, we extend the proof and main theorems of \cite{dong2021attention} to our model, filling in minor errors and gaps in the proof in the literature.

First, recall the defining equation of the residual in \cite{dong2021attention}:
\begin{align}
\boldsymbol{R}^{(\ell)}
&=\textrm{res}(\boldsymbol{X}^{(\ell)})
=\boldsymbol{X}^{(\ell)}-\boldsymbol{1}\boldsymbol{x}^{(\ell)}{}^\top,
\\
\boldsymbol{x}^{(\ell)}
&=
\arg\min{}_{\boldsymbol{x}}
\|
\boldsymbol{X}^{(\ell)}-\boldsymbol{1}\boldsymbol{x}^\top
\|_F,
\end{align}
where a row of rank one matrix for $\boldsymbol{X}$ is given by
\begin{align}
\boldsymbol{x}^{(\ell)}=\frac{1}{N}
\boldsymbol{X}^\top\boldsymbol{1}.
\end{align}

Let's start with the simplest case. 
The single head self-attention layer is
\begin{align}
{\textrm{SA}}(\boldsymbol{X}^{(\ell)})
=\boldsymbol{P}^{(\ell)}\boldsymbol{X}^{(\ell)}
\boldsymbol{W}_V^{(\ell)},
\end{align}
where the attention weight $\boldsymbol{P}^{(\ell)}$ is given by the attention score $\boldsymbol{A}^{(\ell)}$ as
\begin{align}
\boldsymbol{P}^{(\ell)}
=\textrm{softmax}_{\textrm{row}}
\left(\boldsymbol{A}^{(\ell)}\right).
\end{align}
$\textrm{softmax}_{\textrm{row}}$ is the row-wise softmax function.
For self-attention layer,
the residual of the layer is ${\textrm{Res}}\big(
{\textrm{SA}}(\boldsymbol{X}^{(\ell)}
\big)=\boldsymbol{P}^{(\ell)}\boldsymbol{X}^{(\ell)}
\boldsymbol{W}_V^{(\ell)}
-\boldsymbol{1}\boldsymbol{z}^\top$ for certain $\boldsymbol{z}$.
\cite{dong2021attention}
gave inequality for the norm of this residual.


\subsection{Single-Layer MHA}
First, to investigate the rank collapse caused by single MHA-layer, we prove the following inequality to evaluate the degree of convergence of the residual of MHA layer ${\textrm{Res}}\big(
{\textrm{MHA}}(\boldsymbol{X}^{(\ell)})
\big)$:
\begin{tcolorbox}
\begin{lemma}
\label{lem:single-head}
When $\alpha=0$, for a single-head MHA layer, the residual is bounded as follows:
\begin{align}
\label{eq:single-head}
\|
{\textrm{Res}}\big(
{\textrm{MHA}}(\boldsymbol{X}^{(\ell)})
\big)
\|_{1,\infty}
\leq
\frac{4(1-\alpha')C_1}{\sqrt{d_k}}
\|
{\textrm{Res}}\big(
\boldsymbol{X}^{(\ell)}
\big)
\|_{1,\infty}^3
+
\frac{4\alpha' C_2}{\sqrt{d_k}}
\|
{\textrm{Res}}\big(
\boldsymbol{X}^{(\ell)}
\big)
\|_{1,\infty}
.
\end{align}
\end{lemma}
\end{tcolorbox}
In order to examine the contribution of MHA alone, in addition to not adding any shortcut paths, we also removed the shortcut paths within MHA by setting $\alpha=0$.
Extending such inequalities to the case where there are shortcut paths is straightforward by following the techniques in
\cite{dong2021attention}.
The coefficients here are
$C_1=\|
\boldsymbol{W}^{(\ell)}_V
\|_{1,\infty}
\|
\boldsymbol{W}_{QK}^{(\ell)}
\|_1$ and
$C_2=\|
\boldsymbol{W}^{(\ell)}_V
\|_{1,\infty}
\|
{\boldsymbol{A}}^{(\ell-1)}
\|_1$.

The attention weight of our MHA layer is the exponential moving average of query-key dot-product $\boldsymbol{Q}\boldsymbol{K}^\top$ across layers:
\begin{align}
\boldsymbol{P}^{(\ell)}
&=\textrm{softmax}_{\textrm{row}}
\left(\boldsymbol{A}^{(\ell)}\right),
\\
\boldsymbol{A}^{(\ell)}
&=
\frac{1}{\sqrt{d_k}}
\beta
\sum_{m=1}^\ell
\alpha'^{\ell-m}\boldsymbol{Q}^{(m)}\boldsymbol{K}^{(m)}{}^\top
=
\frac{1}{\sqrt{d_k}}
\beta \boldsymbol{Q}^{(\ell)}\boldsymbol{K}^{(\ell)}{}^\top
+
\frac{1}{\sqrt{d_k}}
\alpha'
\boldsymbol{A}^{(\ell-1)},
\end{align}
where $\beta=1-\alpha'$.
The query and key matrices here are
$\boldsymbol{Q}^{(\ell)}
=\boldsymbol{X}^{(\ell)}\boldsymbol{W}_Q^{(\ell)}
+\boldsymbol{1}\boldsymbol{b}_Q^{(\ell)}{}^\top$
and
$
\boldsymbol{K}^{(\ell)}
=\boldsymbol{X}^{(\ell)}\boldsymbol{W}_K^{(\ell)}
+\boldsymbol{1}\boldsymbol{b}_K^{(\ell)}{}^\top$.
Notice that the addition of the bias term 
$+\boldsymbol{1}\boldsymbol{b}_{Q,K}^{(\ell)}{}^\top$
is the broadcast addition of 
$\boldsymbol{b}_{Q,K}^{(\ell)}{}^\top$ along the token dimension.
Substituting this expression gives
\begin{align}
\boldsymbol{A}^{(\ell)}
&=
\frac{1}{\sqrt{d_k}}
\beta \left(\boldsymbol{X}^{(\ell)}\boldsymbol{W}_Q^{(\ell)}
+\boldsymbol{1}\boldsymbol{b}_Q^{(\ell)}{}^\top\right)
\left(
\boldsymbol{X}^{(\ell)}\boldsymbol{W}_K^{(\ell)}
+\boldsymbol{1}\boldsymbol{b}_K^{(\ell)}{}^\top\right)^\top
+
\frac{1}{\sqrt{d_k}}
\alpha'
\boldsymbol{A}^{(\ell-1)}\nonumber\\
&=
\frac{1}{\sqrt{d_k}}
\beta
\boldsymbol{X}^{(\ell)}\boldsymbol{W}_{QK}^{(\ell)}\boldsymbol{X}^{(\ell)}{}^\top
+
\frac{1}{\sqrt{d_k}}
\beta
\boldsymbol{1}\boldsymbol{b}_{QK}^{(\ell)}{}^\top\boldsymbol{X}^{(\ell)}{}^\top
+
\frac{1}{\sqrt{d_k}}
\alpha'
\boldsymbol{A}^{(\ell-1)}
+(\cdots)\boldsymbol{1}^\top
.
\end{align}
We use the following notation
\begin{align}
&\boldsymbol{W}_{QK}^{(\ell)}=
\boldsymbol{W}_Q^{(\ell)}\boldsymbol{W}_K^{(\ell)}{}^\top
,\\
&\boldsymbol{b}_{QK}^{(\ell)}=
\boldsymbol{W}_K^{(\ell)}{}
\boldsymbol{b}_Q^{(\ell)}.
\end{align}
Substituting the definition of the residual, we obtain
$\boldsymbol{X}^{(\ell)}=
\boldsymbol{R}^{(\ell)}+\boldsymbol{1}\boldsymbol{x}^{(\ell)}$

\begin{align}
\boldsymbol{A}^{(\ell)}
=
\frac{1}{\sqrt{d_k}}\left(
\beta
\boldsymbol{1}
\boldsymbol{x}^{(\ell)}{}^\top
\boldsymbol{W}_{QK}^{(\ell)}\boldsymbol{R}^{(\ell)}{}^\top
+
\beta
\boldsymbol{R}^{(\ell)}{}
\boldsymbol{W}_{QK}^{(\ell)}\boldsymbol{R}^{(\ell)}{}^\top
+
\beta
\boldsymbol{1}\boldsymbol{b}_{QK}^{(\ell)}{}^\top\boldsymbol{R}^{(\ell)}{}^\top
+
\alpha'
\boldsymbol{A}^{(\ell-1)}\right)\nonumber\\
+(\cdots)\boldsymbol{1}^\top
.
\end{align}
The point here is that
the last term is constant over row.
Such a constant shift of attention score $\boldsymbol{A}^{(\ell)}$ does not affect the attention weight $\boldsymbol{P}^{(\ell)}$ since the softmax function is row-wise.
Therefore, this term will be omitted below.

To simplify the equations, we introduce the following notation:
\begin{align}
\boldsymbol{r}^{(\ell)}{}^\top
&=
\frac{1}{\sqrt{d_k}}
\beta
\boldsymbol{x}^{(\ell)}{}^\top
\boldsymbol{W}_{QK}^{(\ell)}\boldsymbol{R}^{(\ell)}{}^\top
+
\frac{1}{\sqrt{d_k}}
\beta
\boldsymbol{b}_{QK}^{(\ell)}{}^\top\boldsymbol{R}^{(\ell)}{}^\top
,\\
&=
\frac{1}{\sqrt{d_k}}
\beta
\left(
\boldsymbol{R}^{(\ell)}
\boldsymbol{W}_{QK}^{(\ell)}{}^\top
\boldsymbol{x}^{(\ell)}
+
\boldsymbol{R}^{(\ell)}
\boldsymbol{b}_{QK}^{(\ell)}
\right)^\top,
\\
\label{appendix:defE}
\boldsymbol{E}^{(\ell)}
&=
\frac{1}{\sqrt{d_k}}
\beta
\boldsymbol{R}^{(\ell)}{}
\boldsymbol{W}_{QK}^{(\ell)}
\boldsymbol{R}^{(\ell)}{}^\top.
\end{align}
The attention score is then
\begin{align}
\boldsymbol{A}^{(\ell)}
=
\boldsymbol{E}^{(\ell)}
+
\boldsymbol{1}\boldsymbol{r}^{(\ell)}{}^\top
+
\frac{1}{\sqrt{d_k}}
\alpha'
\boldsymbol{A}^{(\ell-1)}
=
\tilde{\boldsymbol{E}}^{(\ell)}
+
\boldsymbol{1}\boldsymbol{r}^{(\ell)}{}^\top
,
\end{align}
where $\tilde{\boldsymbol{E}}^{(\ell)}
=
{\boldsymbol{E}}^{(\ell)}
+
\frac{1}{\sqrt{d_k}}
\alpha'
\boldsymbol{A}^{(\ell-1)}$.

Since $\boldsymbol{P}^{(\ell)}$ is row-stochastic matrix,
we have
$\boldsymbol{P}^{(\ell)}\boldsymbol{1}
=\boldsymbol{1}$.
Therefore,
$\boldsymbol{P}^{(\ell)}\boldsymbol{R}^{(\ell)}
=
\boldsymbol{P}^{(\ell)}(
\boldsymbol{X}^{(\ell)}
-\boldsymbol{1}\boldsymbol{x}^{(\ell)}{}^\top)$
gives
\begin{align}
\boldsymbol{P}^{(\ell)}\boldsymbol{R}^{(\ell)}
=
\boldsymbol{P}^{(\ell)}\boldsymbol{X}^{(\ell)}
-
\boldsymbol{1}\boldsymbol{x}^{(\ell)}{}^\top.
\end{align}
Using these relations, we can obtain an inequality for the norm of the ${\textrm{Res}}\big({\textrm{MHA}}(\boldsymbol{X}^{(\ell)})\big)$.

First, let us use the inequality used to prove the theorem for self-attention in \cite{dong2021attention}.
Lemma A.3 in \cite{dong2021attention} 
under assumption
$\vert\tilde{E}_{ij}-\tilde{E}_{ik}\vert\leq1.256$
lead to the following element-wise inequalities
\begin{align}
(\boldsymbol{I}-2\tilde{\boldsymbol{D}})
\boldsymbol{1}
\textrm{softmax}(\boldsymbol{r}^{(\ell)}{}^\top)
\boldsymbol{R}^{(\ell)}
\preceq
\boldsymbol{P}^{(\ell)}\boldsymbol{R}^{(\ell)}
\preceq
(\boldsymbol{I}+2\tilde{\boldsymbol{D}})
\boldsymbol{1}
\textrm{softmax}(\boldsymbol{r}^{(\ell)}{}^\top)
\boldsymbol{R}^{(\ell)},
\end{align}
where $\preceq$ is element-wise inequality and $\boldsymbol{I}$ is the identity matrix.
The matrix $\tilde{\boldsymbol{D}}$
is the following diagonal matrix
\begin{align}
\tilde{\boldsymbol{D}}
=
\textrm{diag}(\boldsymbol{d}),\quad
d_i
=
\max_{j,k}
\vert\tilde{E}_{ij}-\tilde{E}_{ik}\vert.
\end{align}
This gives us the following element-wise inequality:
\begin{align}
-2\tilde{\boldsymbol{D}}
\boldsymbol{1}
\textrm{softmax}(\boldsymbol{r}^{(\ell)}{}^\top)
\boldsymbol{R}^{(\ell)}
\preceq
\boldsymbol{P}^{(\ell)}\boldsymbol{R}^{(\ell)}
-
\boldsymbol{1}\left(\boldsymbol{x}^{(\ell)}{}^\top
+
\textrm{softmax}(\boldsymbol{r}^{(\ell)}{}^\top)
\boldsymbol{R}^{(\ell)}
\right)\nonumber\\
\preceq
2\tilde{\boldsymbol{D}}
\boldsymbol{1}
\textrm{softmax}(\boldsymbol{r}^{(\ell)}{}^\top)
\boldsymbol{R}^{(\ell)}.
\end{align}
We therefore have the element-wise inequality
\begin{align}
\label{appendix:elementwise-ineq}
\vert
\boldsymbol{P}^{(\ell)}\boldsymbol{R}^{(\ell)}
-
\boldsymbol{1}\left(\boldsymbol{x}^{(\ell)}{}^\top
+
\textrm{softmax}(\boldsymbol{r}^{(\ell)}{}^\top)
\boldsymbol{R}^{(\ell)}
\right)
\vert
\preceq
2
\vert
\tilde{\boldsymbol{D}}
\boldsymbol{1}
\textrm{softmax}(\boldsymbol{r}^{(\ell)}{}^\top)
\boldsymbol{R}^{(\ell)}
\vert.
\end{align}
Here, 
$\vert\cdot\vert$ is the element-wise absolute value function.
This elementwise inequality implies the following inequalities
for single-head MHA layer
${\textrm{MHA}}(\boldsymbol{X}^{(\ell)})
=
\boldsymbol{P}^{(\ell)}
\boldsymbol{X}^{(\ell)}
\boldsymbol{W}^{(\ell)}_V
$:
\begin{align}
\label{append:ieqL1}
&\|
{\textrm{MHA}}(\boldsymbol{X}^{(\ell)})
-
\boldsymbol{res}^{(\ell)}
\|_{1}
\leq
2
\|
\tilde{\boldsymbol{D}}\,
\boldsymbol{1}\,
\textrm{softmax}(\boldsymbol{r}^{(\ell)}{}^\top)
\boldsymbol{R}^{(\ell)}
\boldsymbol{W}^{(\ell)}_V
\|_{1},\\
\label{append:ieqLinfty}
&\|
{\textrm{MHA}}(\boldsymbol{X}^{(\ell)})
-
\boldsymbol{res}^{(\ell)}
\|_{\infty}
\leq
2
\|
\tilde{\boldsymbol{D}}\,
\boldsymbol{1}\,
\textrm{softmax}(\boldsymbol{r}^{(\ell)}{}^\top)
\boldsymbol{R}^{(\ell)}
\boldsymbol{W}^{(\ell)}_V
\|_{\infty}.
\end{align}
Using the norm properties of (\ref{prelim1}) and (\ref{prelim2}), this inequality can be immediately proven from (\ref{appendix:elementwise-ineq}).
The shorthand notation
$
\boldsymbol{res}^{(\ell)}
=
\boldsymbol{x}^{(\ell)}{}^\top
+
\textrm{softmax}(\boldsymbol{r}^{(\ell)}{}^\top)
\boldsymbol{R}^{(\ell)}
\boldsymbol{W}_V
$
is used in the following.

Using H\"{o}lder's inequality (\ref{holder})
and the submultiplicativity of the 1-norm ,
the right-hand side of
(\ref{append:ieqL1})
becomes
\begin{align}
\|
{\textrm{MHA}}(\boldsymbol{X}^{(\ell)})
-
\boldsymbol{res}^{(\ell)}
\|_{1}
&\leq
2
\|
\tilde{\boldsymbol{D}}\,
\boldsymbol{1}
\|_{\infty}
\|
\textrm{softmax}(\boldsymbol{r}^{(\ell)}{}^\top)
\boldsymbol{R}^{(\ell)}
\boldsymbol{W}^{(\ell)}_V
\|_{1}\nonumber\\
&\leq
2
\|
\tilde{\boldsymbol{D}}\,
\boldsymbol{1}
\|_{\infty}
\|
\boldsymbol{R}^{(\ell)}
\|_{1}
\|
\boldsymbol{W}^{(\ell)}_V
\|_{1}.
\end{align}
The inequality 
$\|\textrm{softmax}(\boldsymbol{r}^{(\ell)}{}^\top)
\|_{1}\leq1$
used to show the second inequality is due to the fact that the magnitude of each element of the row vector on the right side is less than 1.

Similarly, by using submultiplicativity and $\|\textrm{softmax}(\boldsymbol{r}^{(\ell)}{}^\top)
\|_{\infty}=1$, we obtain the following from inequality (\ref{append:ieqLinfty})
\begin{align}
\|
{\textrm{MHA}}(\boldsymbol{X}^{(\ell)})
-
\boldsymbol{res}^{(\ell)}
\|_{\infty}
&\leq
2
\|
\tilde{\boldsymbol{D}}\,
\boldsymbol{1}
\|_{\infty}
\|
\textrm{softmax}(\boldsymbol{r}^{(\ell)}{}^\top)
\|_{\infty}
\|
\boldsymbol{R}^{(\ell)}
\|_{\infty}
\|
\boldsymbol{W}^{(\ell)}_V
\|_{\infty}\nonumber\\
&=
2
\|
\tilde{\boldsymbol{D}}\,
\boldsymbol{1}
\|_{\infty}
\|
\boldsymbol{R}^{(\ell)}
\|_{\infty}
\|
\boldsymbol{W}^{(\ell)}_V
\|_{\infty}.
\end{align}
Multiplying these two inequalities, we obtain
the inequality for the composite of norms $\|\cdot\|_{1,\infty}$ as follows
\begin{align}
\label{appendix:formula1}
\|
{\textrm{MHA}}(\boldsymbol{X}^{(\ell)})
-
\boldsymbol{res}^{(\ell)}
\|_{1,\infty}
\leq
2
\|
\tilde{\boldsymbol{D}}\,
\boldsymbol{1}
\|_{\infty}
\|
\boldsymbol{W}^{(\ell)}_V
\|_{1,\infty}
\|
\boldsymbol{R}^{(\ell)}
\|_{1,\infty}.
\end{align}

Since $\tilde{\boldsymbol{D}}$ depends on the input ${\boldsymbol{X}}^{(\ell)}$, let us evaluate the value of the norm $\|
\tilde{\boldsymbol{D}}\,
\boldsymbol{1}
\|_{\infty}$
in terms of specific inequalities.
From the definitions of $\tilde{\boldsymbol{D}}$, $\tilde{\boldsymbol{E}}$ and the $\infty$-norm, we obtain the following inequality
\begin{align}
\|
\tilde{\boldsymbol{D}}\,
\boldsymbol{1}
\|_{\infty}
&=
\max_{i}
\max_{jk}
\vert
E_{ij}^{(\ell)}-E_{ik}^{(\ell)}
+\alpha' A_{ij}^{(\ell-1)}
-\alpha' A_{ik}^{(\ell-1)}
\vert\nonumber\\
&\leq
2\max_{ij}
\vert
E_{ij}^{(\ell)}
\vert
+
2\alpha'
\max_{ij}
\vert
A_{ij}^{(\ell-1)}
\vert\nonumber\\
&
\label{appendix:formula2}
\leq
2\|
{\boldsymbol{E}}^{(\ell)}
\|_1
+
2
\frac{\alpha'}{\sqrt{d_k}}
\|
{\boldsymbol{A}}^{(\ell-1)}
\|_1.
\end{align}
Using (\ref{appendix:defE}), the norm $\|
{\boldsymbol{E}}^{(\ell)}
\|_1$ becomes
\begin{align}
\|
{\boldsymbol{E}}^{(\ell)}
\|_1
&=
\frac{\beta}{\sqrt{d_k}}
\|
\boldsymbol{R}^{(\ell)}{}
\boldsymbol{W}_{QK}^{(\ell)}
\boldsymbol{R}^{(\ell)}{}^\top\|_1
\nonumber\\
&\leq
\frac{\beta}{\sqrt{d_k}}
\|
\boldsymbol{R}^{(\ell)}{}
\|_1
\|
\boldsymbol{W}_{QK}^{(\ell)}
\|_1
\|
\boldsymbol{R}^{(\ell)}{}^\top\|_1
\nonumber\\
&=
\label{appendix:formula3}
\frac{\beta}{\sqrt{d_k}}
\|
\boldsymbol{W}_{QK}^{(\ell)}
\|_1
\|
\boldsymbol{R}^{(\ell)}{}
\|_{1,\infty}^2
,
\end{align}
where we use
$\|
\boldsymbol{R}^{(\ell)}{}^\top\|_1=\|
\boldsymbol{R}^{(\ell)}\|_\infty$ for the last equality.

Then, by combining (\ref{appendix:formula1}), (\ref{appendix:formula2}), and (\ref{appendix:formula3}), we obtain the following inequality
\begin{align}
\|
{\textrm{MHA}}(\boldsymbol{X}^{(\ell)})
-
\boldsymbol{res}^{(\ell)}
\|_{1,\infty}
\leq
\frac{4\beta}{\sqrt{d_k}}
\|
\boldsymbol{W}^{(\ell)}_V
\|_{1,\infty}
\|
\boldsymbol{W}_{QK}^{(\ell)}
\|_1
\|
\boldsymbol{R}^{(\ell)}{}
\|_{1,\infty}^3\nonumber\\
+
\frac{4\alpha'}{\sqrt{d_k}}
\|
\boldsymbol{W}^{(\ell)}_V
\|_{1,\infty}
\|
{\boldsymbol{A}}^{(\ell-1)}
\|_1
\|
\boldsymbol{R}^{(\ell)}
\|_{1,\infty}
.
\end{align}
Since $\boldsymbol{R}^{(\ell)}={\textrm{Res}}\big(
\boldsymbol{X}^{(\ell)}
\big)$, we obtain Lemma \ref{lem:single-head}:
\begin{tcolorbox}
\begin{align}
\label{eq:shmha}
\|
{\textrm{Res}}\big(
{\textrm{SA}}(\boldsymbol{X}^{(\ell)})
\big)
\|_{1,\infty}
\leq
&\frac{4(1-\alpha')}{\sqrt{d_k}}
\|
\boldsymbol{W}^{(\ell)}_V
\|_{1,\infty}
\|
\boldsymbol{W}_{QK}^{(\ell)}
\|_1
\|
{\textrm{Res}}\big(
\boldsymbol{X}^{(\ell)}
\big)
\|_{1,\infty}^3\nonumber\\
&+
\frac{4\alpha'}{\sqrt{d_k}}
\|
\boldsymbol{W}^{(\ell)}_V
\|_{1,\infty}
\|
{\boldsymbol{A}}^{(\ell-1)}
\|_1
\|
{\textrm{Res}}\big(
\boldsymbol{X}^{(\ell)}
\big)
\|_{1,\infty}
.
\end{align}
\end{tcolorbox}
This inequality is precisely (\ref{eq:single-head}).

\subsection{Single-Layer Multi-Head MHA}

Next, we extend the results of the previous section to the case of multi-headed MHA. The lemma we prove here is as follows:
\begin{tcolorbox}
\begin{lemma}
\label{lem:multi-head}
When $\alpha=0$, for a multi-head MHA layer, the residual is bounded as follows:
\begin{align}
\label{lem-eq:mhmha}
\|
{\textrm{Res}}\big(
{\textrm{MHMHA}}(\boldsymbol{X}^{(\ell)})
\big)
\|_{1,\infty}
\leq
\frac{4H(1-\alpha')C_1}{\sqrt{d_k}}
\|
{\textrm{Res}}\big(
\boldsymbol{X}^{(\ell)}
\big)
\|_{1,\infty}^3
+
\frac{4H\alpha' C_2}{\sqrt{d_k}}
\|
{\textrm{Res}}\big(
\boldsymbol{X}^{(\ell)}
\big)
\|_{1,\infty}
.
\end{align}
\end{lemma}
\end{tcolorbox}
The coefficients used in the above inequality for the multi-head MHA are given by 
$C_1=
\max_h
\|
\boldsymbol{W}^{(\ell)}_{VO,h}
\|_{1,\infty}
\|
\boldsymbol{W}_{QK,h}^{(\ell)}
\|_1$ and
$C_2=
\max_h
\|
\boldsymbol{W}^{(\ell)}_{VO,h}
\|_{1,\infty}
\|
{\boldsymbol{A}}^{(\ell-1)}_h
\|_1$.

Let us generalize (\ref{eq:shmha}) to the case of multi-head MHA.
Multi-head version of MHA layer is
\begin{align}
{\textrm{MHMHA}}(\boldsymbol{X}^{(\ell)})
=
\sum_{h=1}^H
{\boldsymbol{P}}^{(\ell)}_h
\boldsymbol{X}^{(\ell)}
{\boldsymbol{W}}_{VO,h}
+
{\boldsymbol{1}}
{\boldsymbol{b}}_O^\top,
\end{align}
where
${\boldsymbol{W}}_{VO,h}={\boldsymbol{W}}_{V,h}{\boldsymbol{W}}_{O,h}^\top$.
Since this is essentially the sum of single-head contributions,
(\ref{eq:shmha}) gives the following inequality
\begin{align}
\|
{\textrm{Res}}\big(
{\textrm{MHMHA}}(\boldsymbol{X}^{(\ell)})
\big)
\|_{1,\infty}
\leq
&\frac{4(1-\alpha')}{\sqrt{d_k}}
\|
{\textrm{Res}}\big(
\boldsymbol{X}^{(\ell)}
\big)
\|_{1,\infty}^3
\sum_h
\|
\boldsymbol{W}^{(\ell)}_V
\|_{1,\infty}
\|
\boldsymbol{W}_{QK}^{(\ell)}
\|_1\nonumber\\
&+
\frac{4\alpha'}{\sqrt{d_k}}
\|
{\textrm{Res}}\big(
\boldsymbol{X}^{(\ell)}
\big)
\|_{1,\infty}
\sum_h
\|
\boldsymbol{W}^{(\ell)}_V
\|_{1,\infty}
\|
{\boldsymbol{A}}^{(\ell-1)}
\|_1
.
\end{align}
This result immediately gives (\ref{lem-eq:mhmha}).

\subsection{Multi-Layer Multi-Head MHA}

Finally, consider an attention-only network consisting of only multi-head MHA layers
\begin{align}
{\textrm{AttnNet}}(\boldsymbol{X})
={\textrm{MHMHA}}
\circ
{\textrm{MHMHA}}
\circ
\cdots
\circ
{\textrm{MHMHA}}
(\boldsymbol{X}).
\end{align}
The residuals of the network obey the following theorem:
\begin{tcolorbox}
\begin{theorem}
\label{appendix:main-thm}
When $\alpha=0$, for a multi-head MHA-only network, the residual is bounded as follows:
\begin{align}
\|
{\textrm{Res}}&\big(
{\textrm{AttnNet}}(\boldsymbol{X})
\big)
\|_{1,\infty}\nonumber\\
&\leq
\max{}_{m=0}^{L}
\left(
\frac{4H(1-\alpha')C_1}{\sqrt{d_k}}
\right)^{\frac{3^m-1}{2}}
\left(
\frac{4H\alpha' C_2}{\sqrt{d_k}}
\right)^{3^m(L-m)}
\|
{\textrm{Res}}\big(
\boldsymbol{X}
\big)
\|_{1,\infty}^{3^m}
.
\end{align}
\end{theorem}
\end{tcolorbox}
To prove this theorem, we use the following inequality for single-layer MHA, which follows immediately from
(\ref{lem-eq:mhmha}):
\begin{align}
\|
{\textrm{Res}}\big(
{\textrm{MHMHA}}(\boldsymbol{X}^{(\ell)})
\big)
\|_{1,\infty}
\leq
\max\left(
\frac{8H(1-\alpha')C_1}{\sqrt{d_k}}
\|
{\textrm{Res}}\big(
\boldsymbol{X}^{(\ell)}
\big)
\|_{1,\infty}^3
,
\frac{8H\alpha' C_2}{\sqrt{d_k}}
\|
{\textrm{Res}}\big(
\boldsymbol{X}^{(\ell)}
\big)
\|_{1,\infty}
\right)
.
\end{align}
This inequality for MHA has exactly the same structure as the inequality derived in \cite{dong2021attention} when skip connections are introduced to the self-attention layer, so the method of proving the inequality in our AttnNet is the same as in \cite{dong2021attention}.
The point here is that even though our MHA completely eliminates skip connections by setting $\alpha=0$, we can obtain the same improvement in the upper bound of the inequality as when skip connections are added to the normal attention layer.

The proof is as follows: For each layer, expand ${\textrm{Res}}\big(
{\textrm{AttnNet}}(\boldsymbol{X})
\big)$ using (1).
Since the choice of upper bound at each layer is either $
\frac{8H(1-\alpha')C_1}{\sqrt{d_k}}
\|
{\textrm{Res}}\big(
\boldsymbol{X}^{(\ell)}
\big)
\|_{1,\infty}^3$
or
$\frac{8H\alpha' C_2}{\sqrt{d_k}}
\|
{\textrm{Res}}\big(
\boldsymbol{X}^{(\ell)}
\big)
\|_{1,\infty}$,
the calculation of the upper bound of ${\textrm{Res}}\big(
{\textrm{AttnNet}}(\boldsymbol{X})
\big)$ is a repetition of these two choices. Therefore, in the tree obtained by repeated binary expansion over $L$ layers, we need to consider the maximum value of the leaf term, which is $\max
\left(
\frac{4H(1-\alpha')C_1}{\sqrt{d_k}}
\right)^{\frac{3^m-1}{2}}
\left(
\frac{4H\alpha' C_2}{\sqrt{d_k}}
\right)^{3^m(L-m)}
\|
{\textrm{Res}}\big(
\boldsymbol{X}^{(0)}
\big)
\|_{1,\infty}^{3^m}$.

\section{Empirical Check of Theorem}

We experimentally demonstrate that the upper bound inequality derived in this paper accurately explains the decaying behavior observed in actual networks.
The effect of the doubly exponential decay in norms becomes particularly prominent in deep models.
To provide a clear and direct setting for comparison with the theoretical analysis, we therefore conduct evaluations in multi-layer configurations.
Following \cite{ramsauerhopfield}, we evaluate the following metric at initialization:
\begin{align}
\frac{\|\textrm{Res}(\textrm{AttnNet}(\boldsymbol{X}))\|_{1,\infty}}{\|\textrm{AttnNet}(\boldsymbol{X})\|_{1,\infty}}.
\end{align}

Table \ref{tab:attn-only} shows the layer-wise variation of this normalized norm (metric).
\begin{table}[ht]
    \centering
    \begin{tabular}{c|cc}
        \hline
        depth of network           & self-attention       & MHA($\alpha=0.5$)      \\
        \hline\hline
        1          & $0.48887348$      & $0.8797747$  \\
        \hline
        2          & $0.07732825$      & $0.8309065$  \\
        \hline
        3          & $0.00081001734$      & $0.80018705$  \\
        \hline
        4          & $1.7725031*10^{-6} $      & $0.7630522$  \\
        \hline
        5          & $1.809994*10^{-6}$      & $0.71199465 $  \\
        \hline
        6          & $1.7835139*10^{-6}$      & $ 0.67957425$  \\
        \hline
        7          & $1.7920005*10^{-6}$      & $0.6251911 $  \\
        \hline
        8          & $1.7568021*10^{-6}$      & $ 0.5860811$  \\
        \hline
        9          & $1.75397*10^{-6}$      & $ 0.52244353$  \\
        \hline
        10          & $1.7860738*10^{-6}$      & $ 0.46130562$  \\
        \hline
        11         & $1.82849*10^{-6}$      & $0.432442 $  \\
        \hline
        12         & $1.7818496*10^{-6}$      & $0.39708787 $  \\
    \end{tabular}
    \caption{Normalized norm across layers in Attention-Only (ViT-T) and MHA networks evaluated on CIFAR-10 at initialization.}
    \label{tab:attn-only}
\end{table}
We evaluated this norm on the CIFAR-10 dataset for Attention-Only Networks with 1 to 12 layers (ViT-T configuration) and their MHA counterparts.
The left column shows the results with standard self-attention. As observed in previous studies, we confirm a sharp norm decay across layers, corresponding to rank collapse. In contrast, the right column shows that introducing MHA significantly mitigates this norm decay and suppresses the rank collapse through layers.

While the inequality we derive in this paper provides a theoretical upper bound on the norm decay, it closely aligns with the empirical behavior, and captures the rank-collapse suppression effect of MHA.

\section{Choices of Discretization}

\begin{figure}[ht]
\begin{center}
\centerline{\includegraphics[width=0.6\columnwidth]{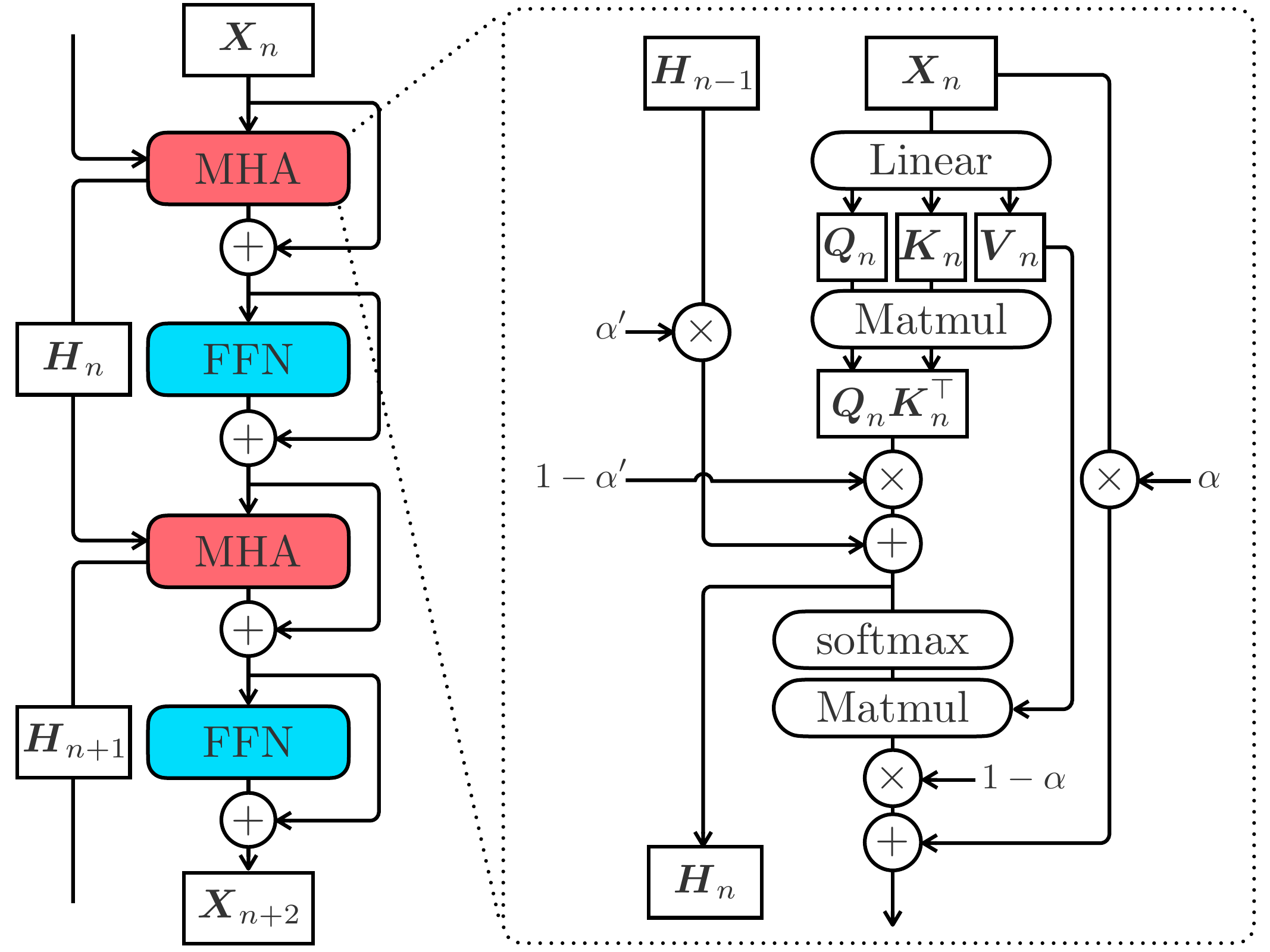}}
\caption{The architecture of the MHA examined in detail in this paper. This model corresponds to the case where the forward derivative is used for the visible state and the backward derivative for the hidden state. Simply setting $\alpha=\alpha'=0$ reproduces normal self-attention.}
\label{fig-orig-archi}
\end{center}
\end{figure}

In the paper, we employ the MHA architecture illustrated in Figure \ref{fig-orig-archi}.
However, different types of architectures can be introduced depending on the choice of discretization. 
In this section, we explain what results are obtained when the time evolution differential equations for visible and hidden states are discretized using forward or backward differentiation. 
We then comment on the reasons why we specifically address the Figure \ref{fig-orig-archi} case in this paper.

First, recall the state evolution equation of MCHN for Model B:
\begin{align}
\label{eq:mchnMB1}
&\tau_v\frac{d\boldsymbol{x}}{dt}
=\textrm{softmax}\left(\boldsymbol{h}\right)\boldsymbol{V}-\boldsymbol{x},\\
&
\label{eq:mchnMB2}
\tau_h\frac{d\boldsymbol{h}}{dt}
=\boldsymbol{q}\boldsymbol{K}^\top-\boldsymbol{h},
\end{align}
Here, we have rewritten the expression with key $\boldsymbol{K}$, query $\boldsymbol{q}$, and value $\boldsymbol{V}$ using the correspondence with Transformer \cite{ramsauerhopfield}.
We will use forward or backward differentiation for the discretization of the left hand side of these equations.

\subsection{The Case of (Forward, Forward)}

First, let us consider the case where both visible and hidden states are discretized by forward differentiation. In this case, the differential equations (\ref{eq:mchnMB1},\ref{eq:mchnMB2}) becomes the following difference equations
\begin{align}
\label{eq:ff1}
&\frac{1}{\rho}
\left(
\boldsymbol{x}_{n+1}-\boldsymbol{x}_{n}
\right)
=\textrm{softmax}\left(\boldsymbol{h}_n\right)\boldsymbol{V}_n-\boldsymbol{x}_n,\\
&
\label{eq:ff2}
\frac{1}{\rho'}
\left(
\boldsymbol{h}_{n+1}-\boldsymbol{h}_{n}
\right)
=\boldsymbol{q}_n\boldsymbol{K}_n^\top-\boldsymbol{h}_n,
\end{align}
where we use the following notation
\begin{align}
\frac{\tau_v}{\Delta t}=\frac{1}{\rho},\quad
\frac{\tau_h}{\Delta t}=\frac{1}{\rho'}.
\end{align}

These equations are organized in a form that is easy to compare with self-attention as follows:
\begin{align}
\label{eq:ff1-2}
&
\boldsymbol{x}_{n+1}=
(1-\rho)\boldsymbol{x}_{n}
+
\rho\,
\textrm{softmax}\left(\boldsymbol{h}_n\right)\boldsymbol{V}_n,\\
&
\label{eq:ff2-2}
\boldsymbol{h}_{n}
=(1-\rho')\boldsymbol{h}_{n-1}
+\rho'\boldsymbol{q}_{n-1}\boldsymbol{K}_{n-1}^\top.
\end{align}
This unfamiliar form of this attention-like model calculates the attention weights
$\textrm{softmax}(\boldsymbol{h}_n)$
using information from the attention scores $\boldsymbol{q}_{n-1}\boldsymbol{K}_{n-1}^\top$
up to one layer prior through the hidden state, as shown in Figure \ref{fig-ff-archi}. Thus, the attention weights applied to values $\boldsymbol{V}_n$
do not include information about queries and keys in the same layer. Therefore, even if we consider the case $\rho'=1$ in (\ref{eq:ff2-2}), we do not return to the usual attention, but only to an attention score of $0$. 

Mathematically, the usual self-attention is obtained using the equation
$0=\boldsymbol{h}_{n-1}+\boldsymbol{q}_{n-1}\boldsymbol{K}_{n-1}^\top$, which is obtained by taking the limit  $\rho'\to\infty$ in (\ref{eq:ff2-2}).
This is precisely the adiabatic limit of \cite{dong2021attention}.
However, here we consider these state update equations as forward propagation equations of neural network.
In actual network implementations, it is difficult to naturally perform operations equivalent to this adiabatic limit for the model. Therefore, since this model does not naturally give Transformer, we do not examine it in detail in this paper.

\begin{figure}[ht]
\begin{center}
\centerline{\includegraphics[width=0.6\columnwidth]{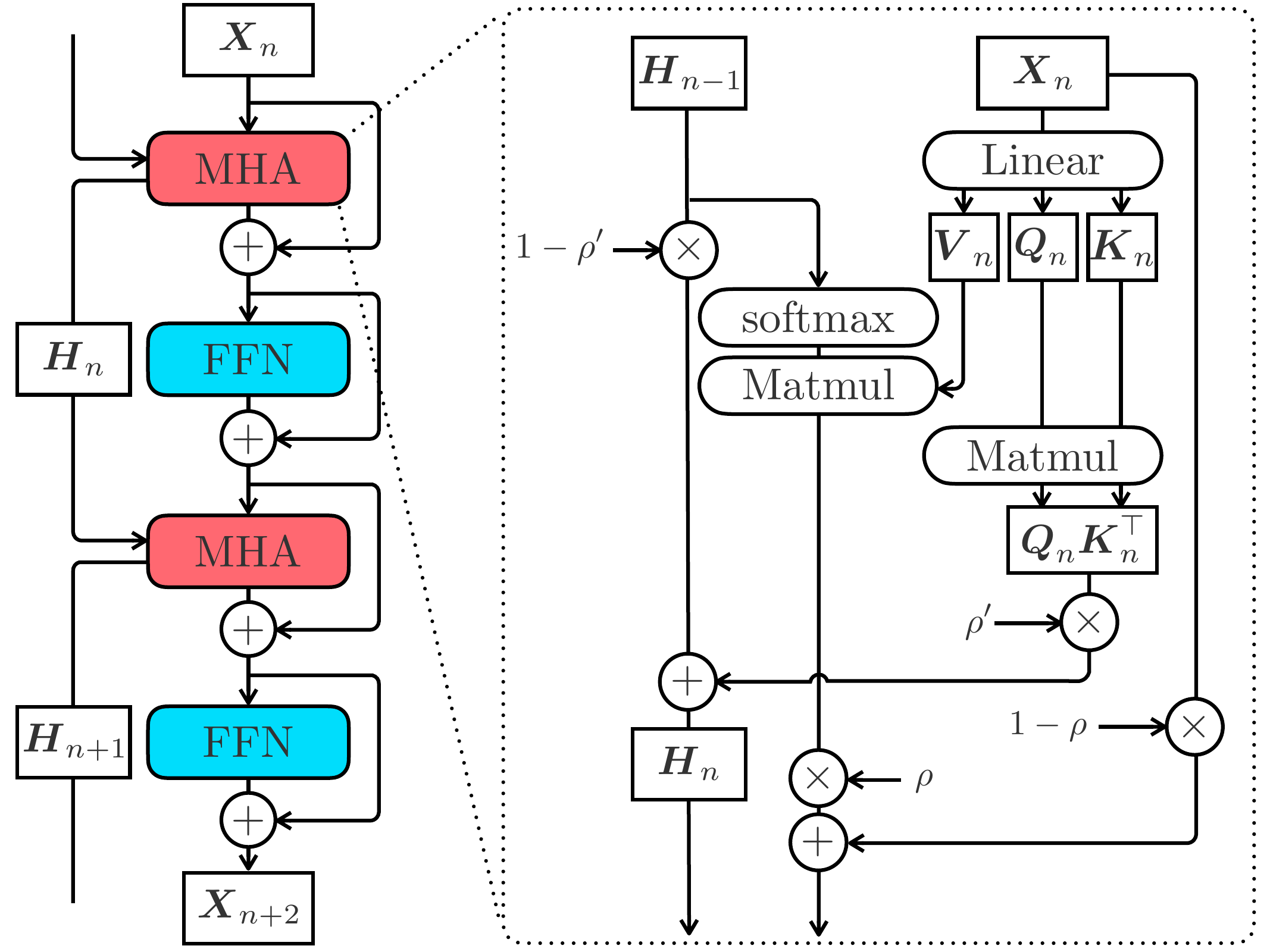}}
\caption{The architecture corresponds to the case where the forward derivative is used for the visible state and the hidden state.}
\label{fig-ff-archi}
\end{center}
\end{figure}

\subsection{The Case of (Forward, Backward)}

Next, let us consider the case in which the forward derivative is used for the visible state while the backward derivative is used for the hidden state. This is the case examined in detail in this paper. In this case, the evolution equation for the hidden state is as follows
\begin{align}
\label{eq:fb2}
\frac{1}{\rho'}
\left(
\boldsymbol{h}_{n+1}-\boldsymbol{h}_{n}
\right)
=\boldsymbol{q}_{n+1}\boldsymbol{K}_{n+1}^\top-\boldsymbol{h}_{n+1}.
\end{align}
These equations can be organized in a form that is easy to compare with self-attention as follows:
\begin{align}
\label{eq:fb1-2}
&
\boldsymbol{x}_{n+1}=
(1-\rho)\boldsymbol{x}_{n}
+
\rho\,
\textrm{softmax}\left(\boldsymbol{h}_n\right)\boldsymbol{V}_n,\\
&
\label{eq:fb2-2}
\boldsymbol{h}_{n}
=\frac{1}{1+\rho'}\boldsymbol{h}_{n-1}
+\frac{\rho'}{1+\rho'}
\boldsymbol{q}_{n}\boldsymbol{K}_{n}^\top.
\end{align}
By setting $\alpha=1-\rho$ and $\alpha'=\frac{1}{1+\rho'}$, we see that these equations give Figure \ref{fig-orig-archi}.
The important point in this case is that this architecture includes usual self-attention in a natural way.
It is easy to see that $\alpha=\alpha'=0$ reproduces the original self-attention, so this can be seen as a natural extension of the Transformer.
Therefore, this paper, whose purpose is to examine the implications of MCHN for Transformer, has studied this case in detail.

\subsection{The Case of (Backward, Forward)}

Next, let us consider the case in which the backward derivative is used for the visible state while the forward derivative is used for the hidden state.
In this case, the evolution equation for the visible state is as follows
\begin{align}
\label{eq:bf1}
\frac{1}{\rho}
\left(
\boldsymbol{x}_{n+1}-\boldsymbol{x}_{n}
\right)
=\textrm{softmax}\left(\boldsymbol{h}_{n+1}\right)\boldsymbol{V}_{n+1}-\boldsymbol{x}_{n+1}.
\end{align}
We can organize these equations in the following form, which is easy to compare with self-attention:
\begin{align}
\label{eq:bf1-2}
&
\boldsymbol{x}_{n+1}=
\frac{1}{1+\rho}
\boldsymbol{x}_{n}
+
\frac{\rho}{1+\rho}\,
\textrm{softmax}\left(\boldsymbol{h}_{n+1}\right)\boldsymbol{V}_{n+1},\\
&
\label{eq:bf2-2}
\boldsymbol{h}_{n+1}
=(1-\rho')\boldsymbol{h}_{n}
+\rho'\boldsymbol{q}_{n}\boldsymbol{K}_{n}^\top.
\end{align}
However, a problem arises when we try to consider the state update equation for the visible state as a forward propagation equation.
This is because $\boldsymbol{V}_{n+1}$ is needed to calculate $\boldsymbol{x}_{n+1}$, and $\boldsymbol{V}_{n+1}$ is given by the linear projection of $\boldsymbol{x}_{n+1}$. 
Therefore, this equation cannot be naturally interpreted as a feedforward neural network as it is. 
We mention the possibility of interpreting it as a recurrent network in the next example, but we will not consider this case in depth in this paper.

\subsection{The Case of (Backward, Backward)}

\begin{figure}[ht]
\begin{center}
\centerline{\includegraphics[width=0.6\columnwidth]{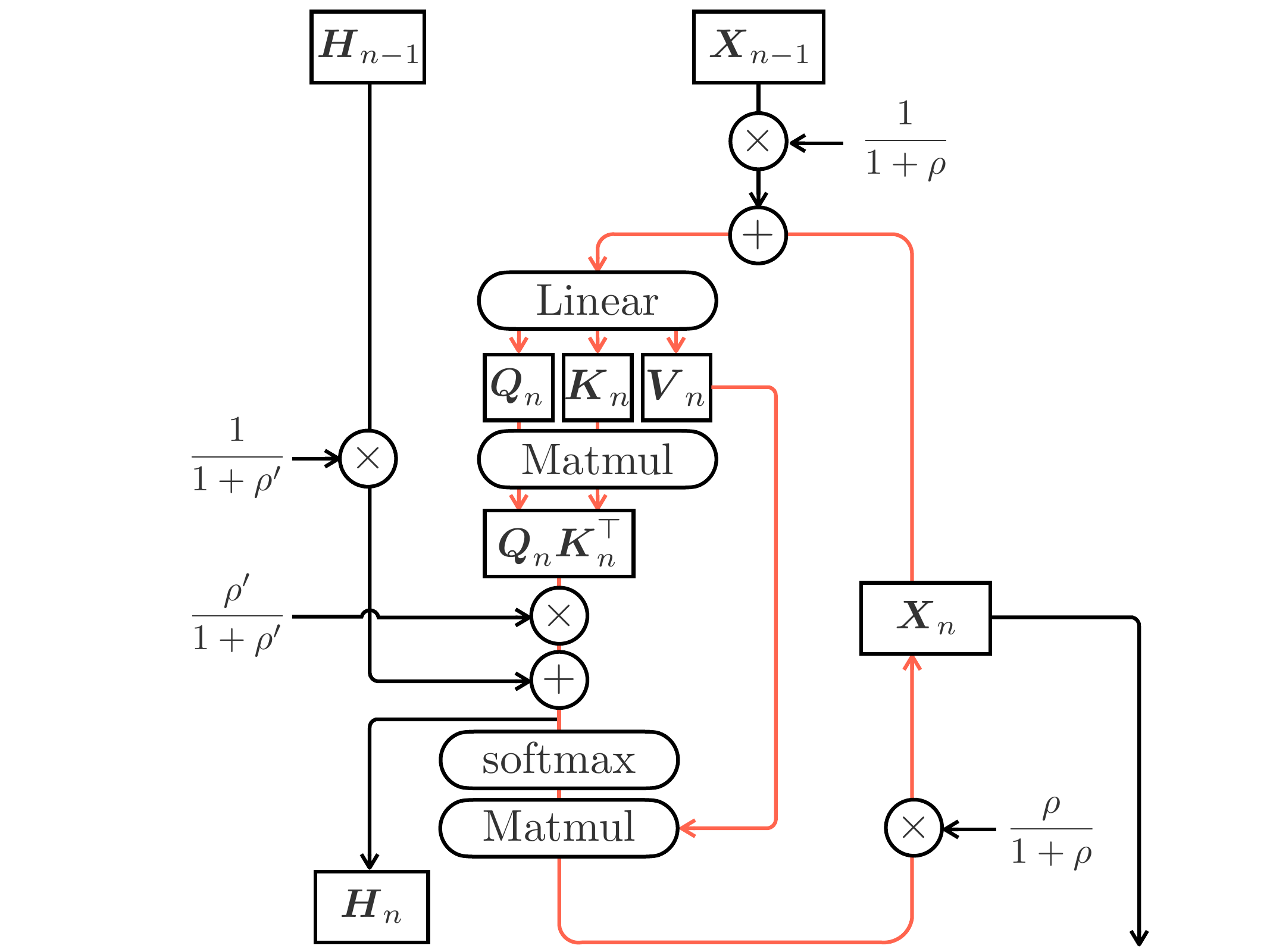}}
\caption{Possible interpretations as a recurrent neural net when both visible and hidden states use backward differentiation.}
\label{fig-bb-archi}
\end{center}
\end{figure}

Finally, consider the case where both visible and hidden states are discretized by backward differentiation. In this case, the time evolution equations for these two states are as follows
\begin{align}
\label{eq:bb1-2}
&
\boldsymbol{x}_{n+1}=
\frac{1}{1+\rho}
\boldsymbol{x}_{n}
+
\frac{\rho}{1+\rho}\,
\textrm{softmax}\left(\boldsymbol{h}_{n+1}\right)\boldsymbol{V}_{n+1},,\\
&
\label{eq:bb2-2}
\boldsymbol{h}_{n+1}
=\frac{1}{1+\rho'}\boldsymbol{h}_{n}
+\frac{\rho'}{1+\rho'}
\boldsymbol{q}_{n+1}\boldsymbol{K}_{n+1}^\top.
\end{align}
This equation (\ref{eq:bb1-2}) cannot be naturally interpreted as a feedforward neural network as in the previous case. 
If we try to understand it as a neural network, it may be possible to interpret it as a recurrent neural network, as shown in Figure \ref{fig-bb-archi}. In any case, such complicated cases are beyond the scope of this paper.
Therefore, we will not examine it in depth in this paper.

\subsection{The Cases of Central Differentiation}

In addition to the previous examples, there is also the option of using central derivatives.
Using the central derivative for the visible state yields the following update rule
\begin{align}
\boldsymbol{x}_{n+1}=
2{\rho}\,
\textrm{softmax}\left(\boldsymbol{h}_{n}\right)\boldsymbol{V}_{n}
+
2{\rho}\,\boldsymbol{x}_{n}
+
\boldsymbol{x}_{n-1}.
\end{align}
In addition to $2{\rho}\,\boldsymbol{x}_{n}$, there is also a shortcut path that originates from $\boldsymbol{x}_{n-1}$ in this forward propagation formula.
This implies DenseNet-like generalization of shortcut path of Transformer and MHA.

On the other hand, using the central derivative for the hidden state yields the following equation
\begin{align}
\boldsymbol{h}_{n+1}
=
2\rho'\boldsymbol{q}_{n}\boldsymbol{K}_{n}
+
2\rho'\boldsymbol{h}_{n}
+\boldsymbol{h}_{n-1}.
\end{align}
This forward propagation formula complicates the computation of the attention weights because even information from two previous layers contributes to the calculation of the hidden state.
Such overly complex architectures are outside the scope of this paper's investigation.

\section{Experimental Setups}

In this section, we describe the detailed training settings used in the experiments in this paper.

We used two sizes of GPT-2 models for text generation task in this paper, and their detailed information and training settings are summarized in Table \ref{tab:gpt_config}.

 \begin{table}[]
     \centering
     \begin{tabular}{c|c}
         \hline
         parameter                 & Small/Medium \\
         \hline\hline
         num heads                 &  12/16         \\
         depth                     &   12/24        \\
         optimizer                 & AdamW                 \\
         base lerning rate         & 6e-4               \\
         batch size                & 12/4        \\
         num tokens & 1024 \\
         embedding dimention       & 768/1024      \\
         learning rate schedule    & cosine decay          \\
         lower learning rate bound & 1e-6                  \\
         warmup epochs             & 20                    \\
         warmup schedule           & linear                \\
         warmup learning rate      & 1e-6                  \\
     \end{tabular}
     \vskip 0.1in
     \caption{Detailed information on the GPT2 models and the settings used to train them.}
     \label{tab:gpt_config}
 \end{table}

We used four sizes of ViT models for image recognition task in this paper, and their detailed information and training settings are summarized in Table \ref{tab:cv_config}.

 \begin{table}[hbt]
     \centering
     \begin{tabular}{c|c}
         \hline
         parameter                 & Tiny/Small/Base/Large \\
         \hline\hline
         num heads                 & 3/6/12/16             \\
         depth                     & 12/12/12/24           \\
         droppath rate             & 0.1/0.2/0.2/0.3       \\
         optimizer                 & AdamW                 \\
         optimizer $\epsilon$      & 1e-8                  \\
         training epochs           & 300    \\
         base lerning rate         & 6.25e-5               \\
         batch size                & 256/256/128/32        \\
         patch size                & 16                    \\
         embedding dimention       & 192/384/768/1024      \\
         resized image size        & (3,224,224)           \\
         learning rate schedule    & cosine decay          \\
         lower learning rate bound & 1e-6                  \\
         warmup epochs             & 20                    \\
         warmup schedule           & linear                \\
         warmup learning rate      & 1e-6                  \\
         cooldown epochs           & 10                    \\
         random erasing            & 0.25                  \\
         mixup $\alpha$            & 0.8                   \\
         cutmix $\alpha$           & 1.0                   \\
         label smoothing           & 0.1                   \\
         RandAugment               & (9,0.5)               \\
     \end{tabular}
     \vskip 0.1in
     
     \caption{Detailed information on the ViT models and the settings used to train them.}
     \label{tab:cv_config}
 \end{table}
 
 
\section{Rank Collapse}

In this section, we provide detailed information about the token similarity at each layer of GPT and ViT, and also present experimental results that could not be fully presented in the main text.

The rank collapse in \cite{dong2021attention} refers to the phenomenon in which the tokens corresponding to each row become perfectly proportional vectors in the attention network features, and the features degenerate into a matrix of rank $1$. This means perfect token uniformity, where the cosine similarity of all token pairs is $1$. On the other hand, the phenomenon observed in the actual Transformer is that although not all tokens are perfectly aligned, many tokens are perfectly aligned, resulting in the formation of a token population with a mutual cosine similarity of $1$. In terms of rank, this is partial rank collapse, where the features degenerate into a matrix with a lower but non-$1$ rank. Here, we observe this phenomenon through detailed experimental results.

Figure \ref{fig:cossim_wiki_medium} are plots that compare the token cosine similarity of GPT-2 (Medium) in each layer. In original GPT-2, there exists a population of tokens with a cosine similarity of $1$. On the other hand, in the model using MHA, such token uniformity disappears. This result shows that MHA has a strong function of suppressing rank collapse.

\begin{figure}[ht]
\begin{center}
\centerline{\includegraphics[width=0.72\columnwidth]{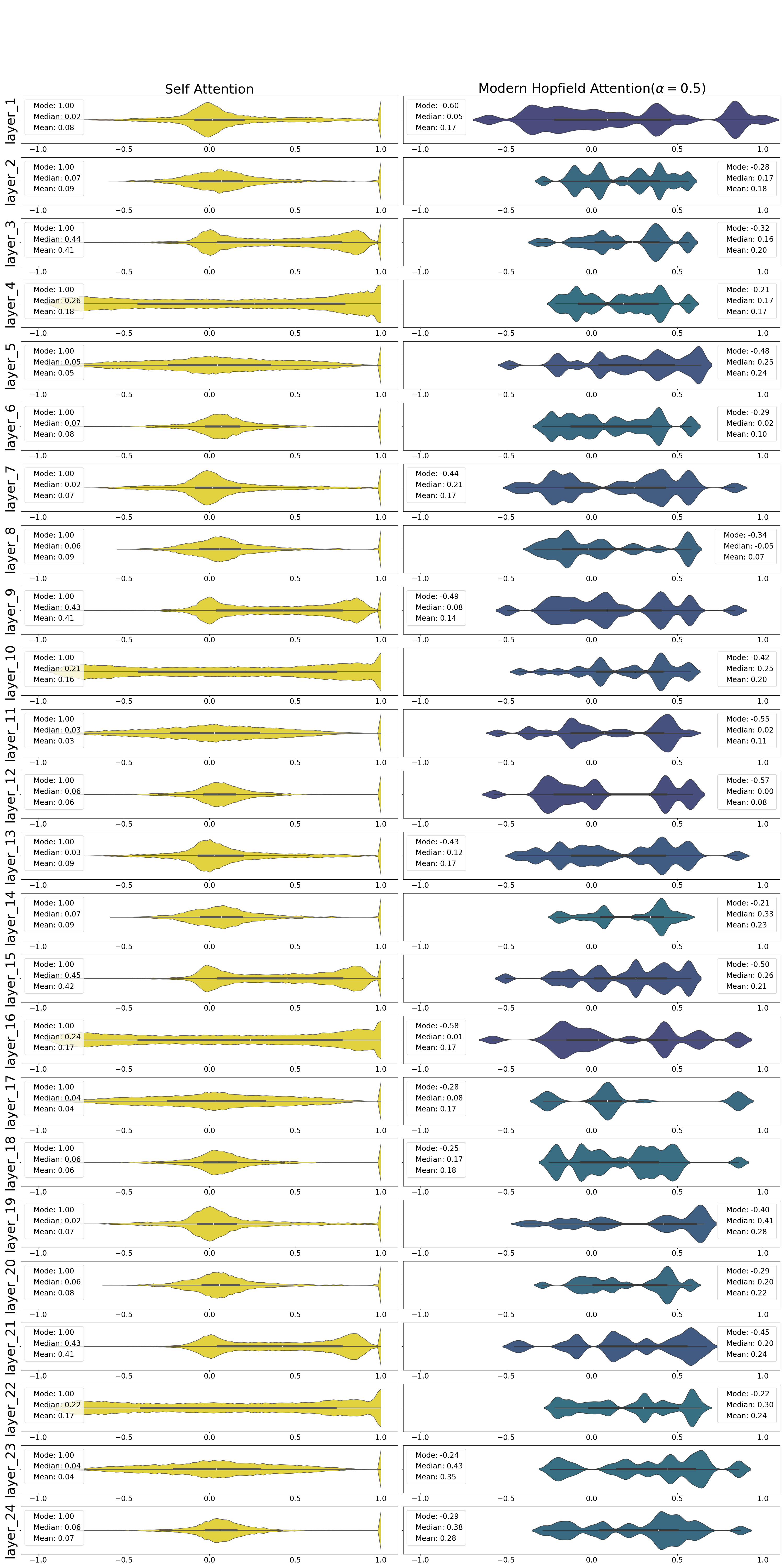}}
\caption{The violin plots of the cosine similarity of GPT-2 (Medium) with usual self-attention
and  MHA for $\alpha=0.5$.
The model is trained with Wikitext103.
We can see that the group of perfectly aligned tokens that exists at a peak around a similarity of $1$ in self-attention disappears in the MHA cases.}
\label{fig:cossim_wiki_medium}
\end{center}
\end{figure}

Figures \ref{fig:cossim_cifar10_tiny}-\ref{fig:cossim_cifar100_large} are detailed plots comparing the token cosine similarity of ViT in each layer. In each case, in the normal ViT, there exists a population of tokens with a cosine similarity of 1 even after training. On the other hand, in the model using MHA, such token uniformity disappears. This result shows that MHA has a strong function of suppressing rank collapse.


\begin{figure}[ht]
\begin{center}
\centerline{\includegraphics[width=0.95\columnwidth]{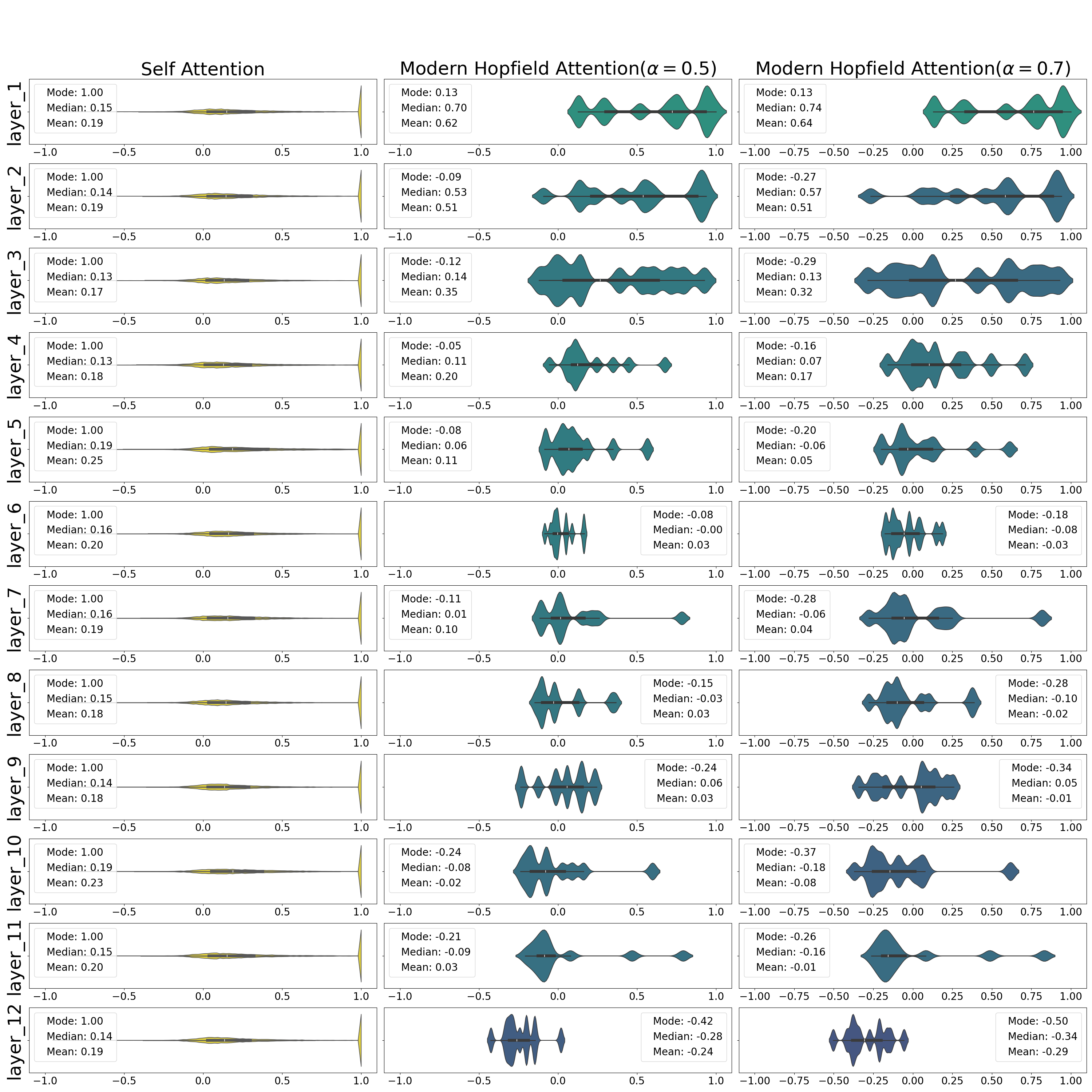}}
\caption{The violin plots of the cosine similarity of ViT-T with usual self-attention, MHA for $\alpha=0.5$
and  MHA for $\alpha=0.7$.
The model is trained with CIFAR10.
We can see that the group of perfectly aligned tokens that exists at a peak around a similarity of $1$ in self-attention disappears in the MHA cases.}
\label{fig:cossim_cifar10_tiny}
\end{center}
\end{figure}

\begin{figure}[ht]
\begin{center}
\centerline{\includegraphics[width=0.95\columnwidth]{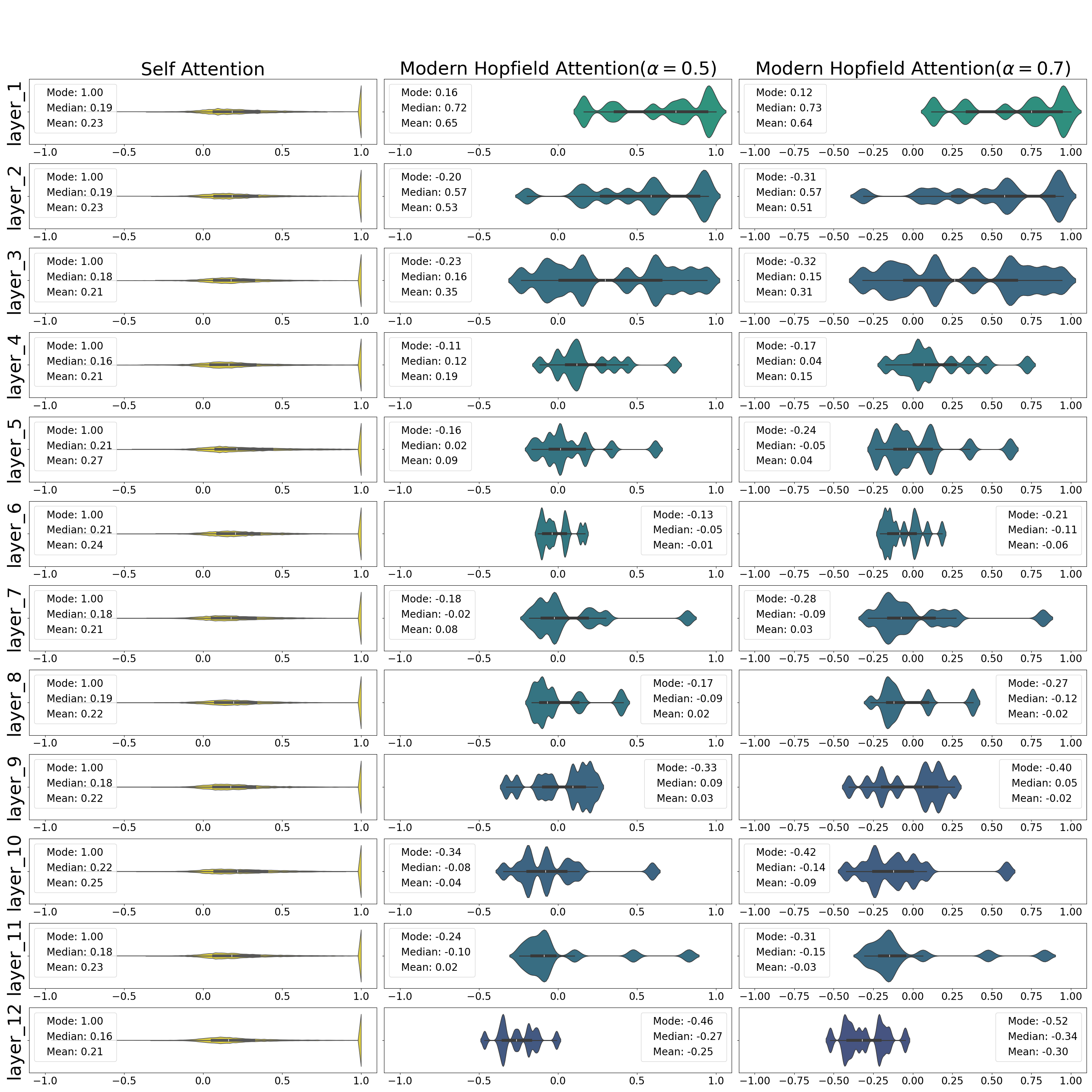}}
\caption{The violin plots of the cosine similarity of ViT-S with usual self-attention, MHA for $\alpha=0.5$
and  MHA for $\alpha=0.7$.
The model is trained with CIFAR10.}
\label{fig:cossim_cifar10_small}
\end{center}
\end{figure}

\begin{figure}[ht]
\begin{center}
\centerline{\includegraphics[width=0.95\columnwidth]{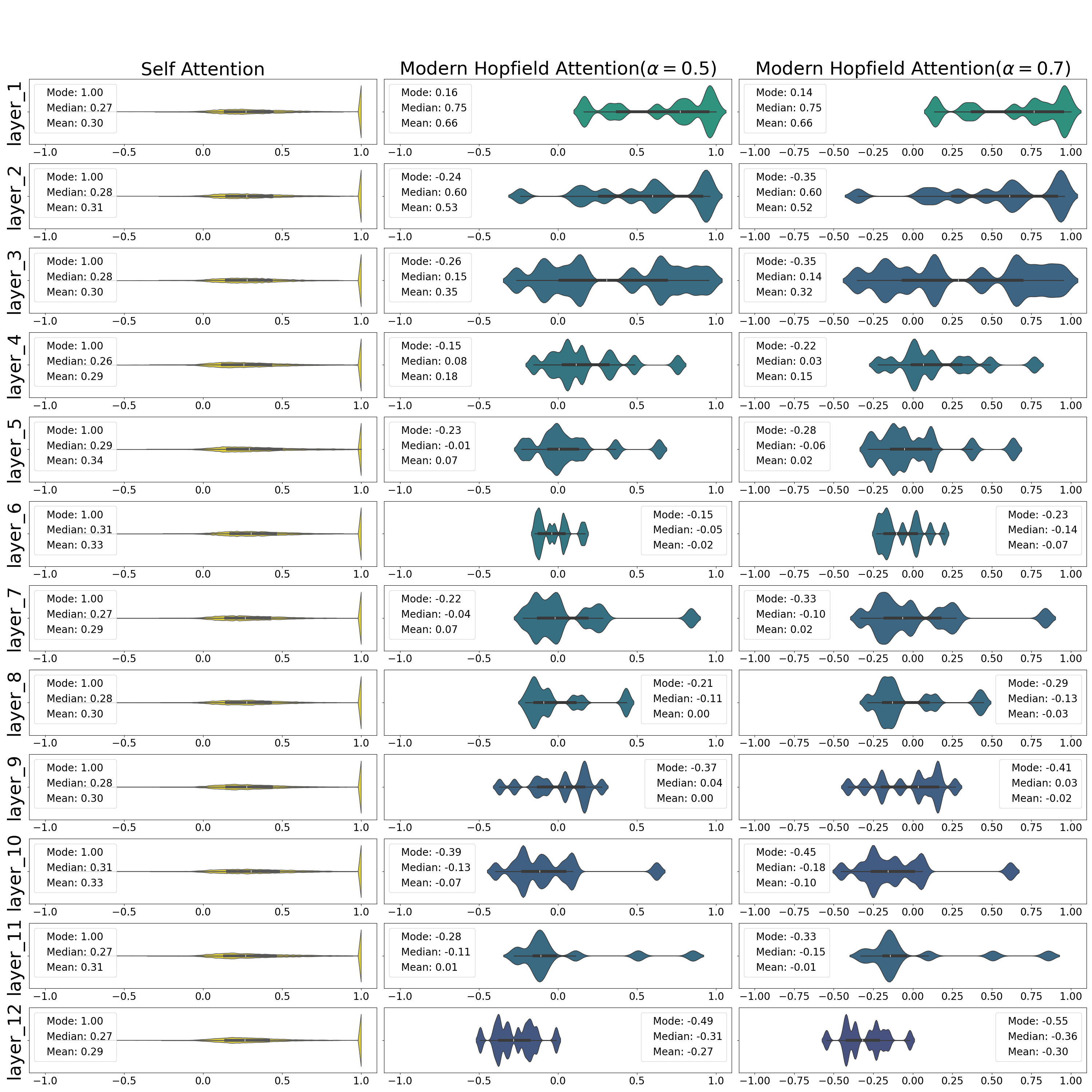}}
\caption{The violin plots of the cosine similarity of ViT-B with usual self-attention, MHA for $\alpha=0.5$
and  MHA for $\alpha=0.7$.
The model is trained with CIFAR10.}
\label{fig:cossim_cifar10_base}
\end{center}
\end{figure}

\begin{figure}[ht]
\begin{center}
\centerline{\includegraphics[width=0.75\columnwidth]{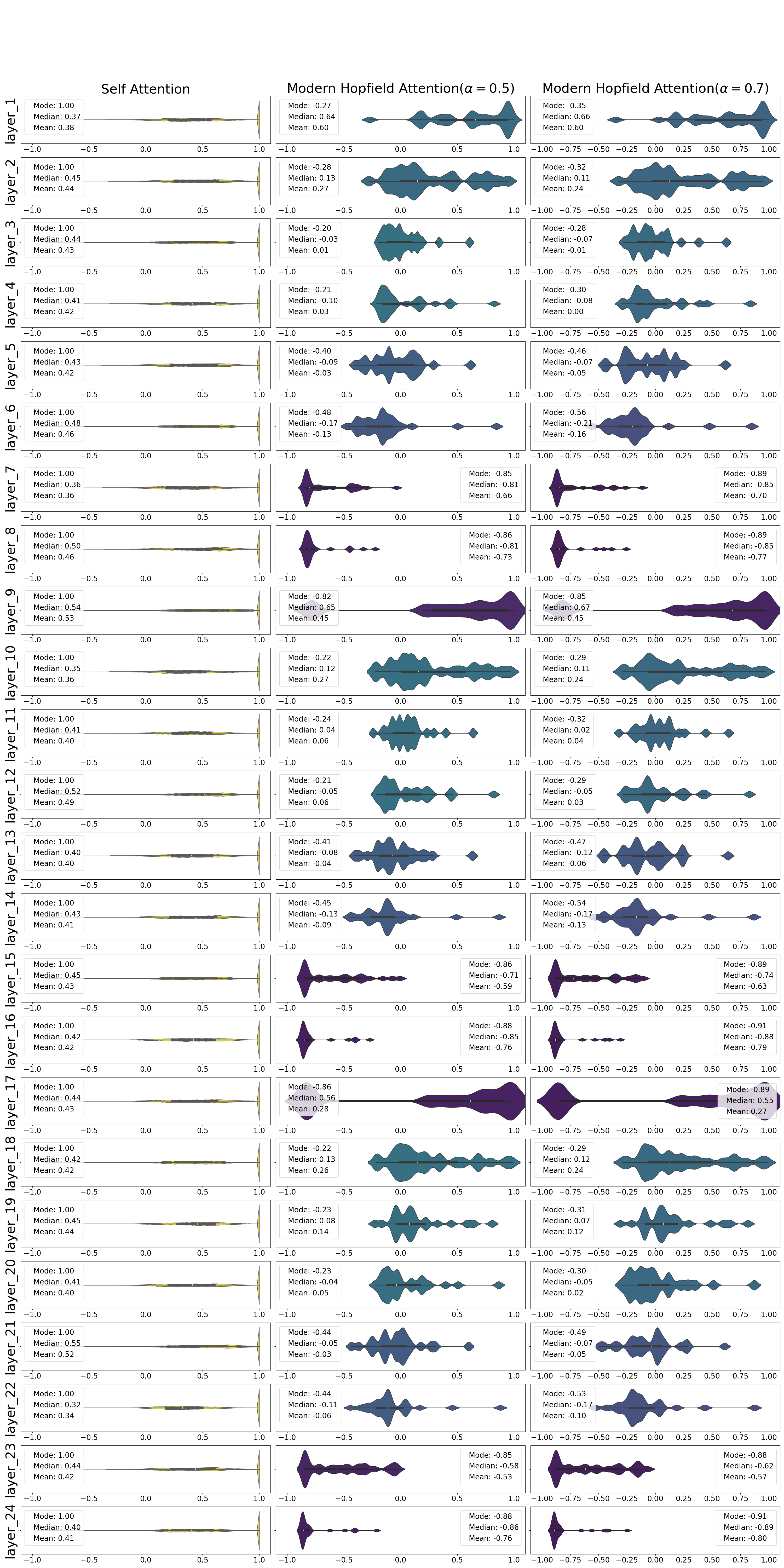}}
\caption{The violin plots of the cosine similarity of ViT-L with usual self-attention, MHA for $\alpha=0.5$
and  MHA for $\alpha=0.7$.
The model is trained with CIFAR10.}
\label{fig:cossim_cifar10_large}
\end{center}
\end{figure}

\begin{figure}[ht]
\begin{center}
\centerline{\includegraphics[width=0.95\columnwidth]{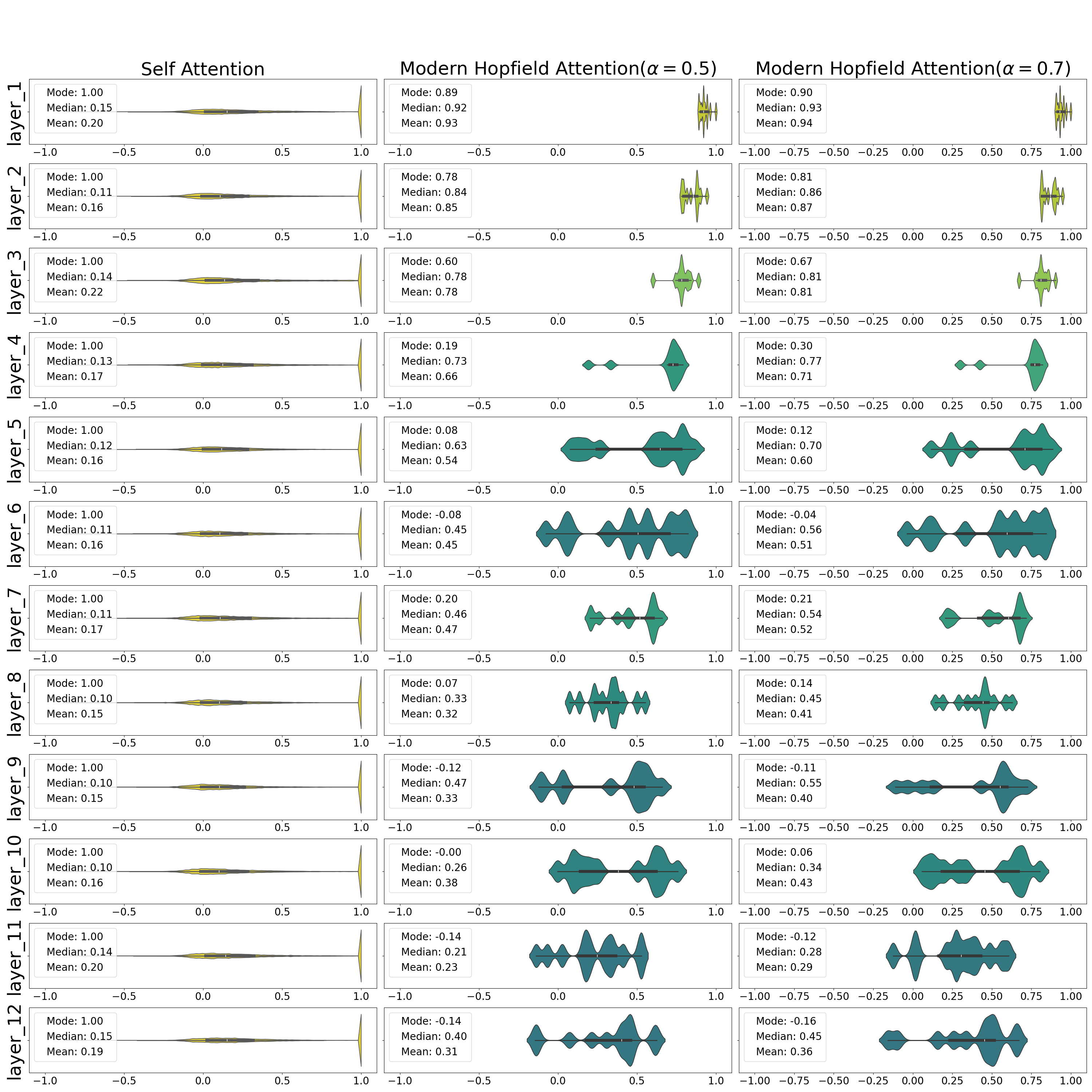}}
\caption{The violin plots of the cosine similarity of ViT-T with usual self-attention, MHA for $\alpha=0.5$
and  MHA for $\alpha=0.7$.
The model is trained with CIFAR100.
We can see that the group of perfectly aligned tokens that exists at a peak around a similarity of $1$ in self-attention disappears in the MHA cases.}
\label{fig:cossim_cifar100_tiny}
\end{center}
\end{figure}

\begin{figure}[ht]
\begin{center}
\centerline{\includegraphics[width=0.95\columnwidth]{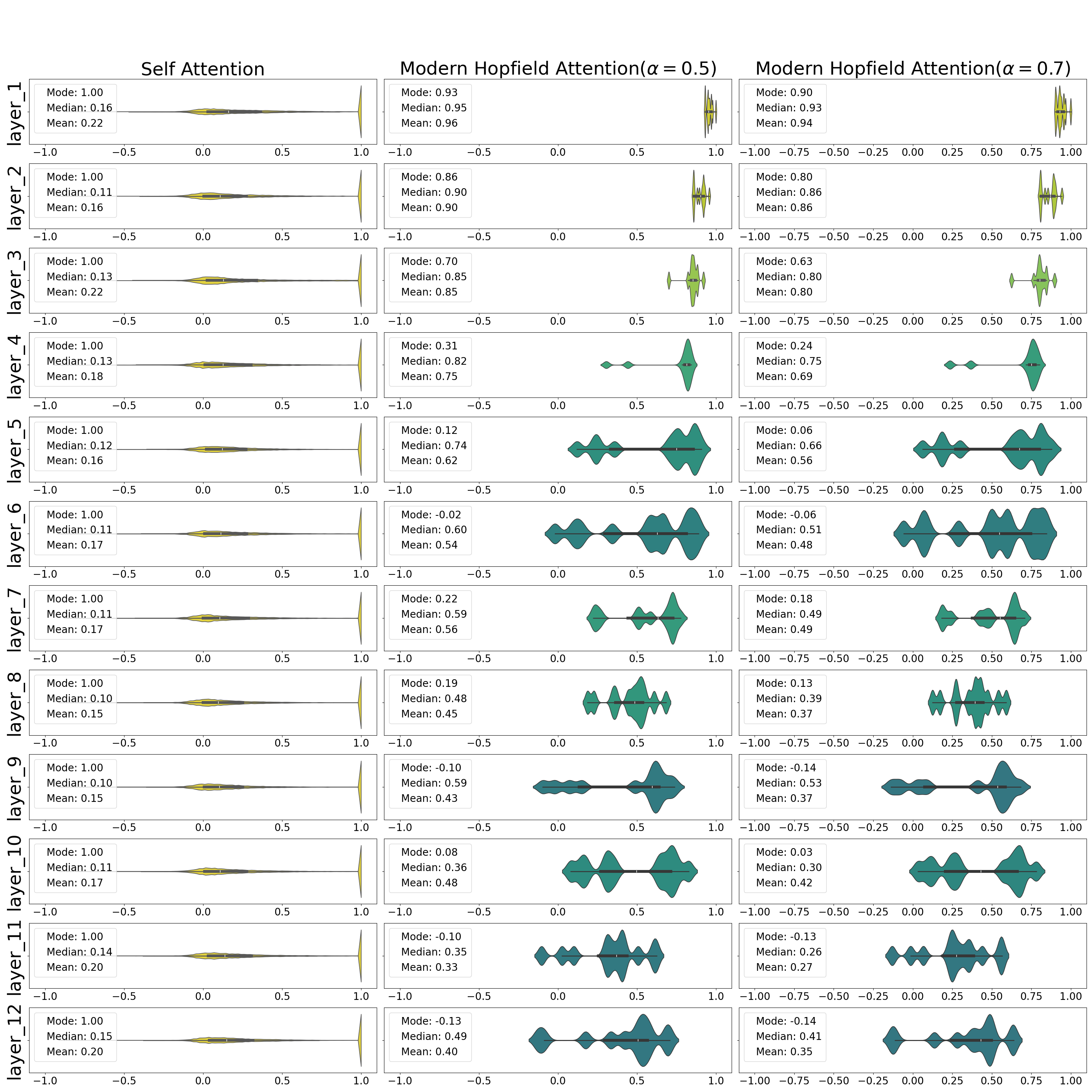}}
\caption{The violin plots of the cosine similarity of ViT-S with usual self-attention, MHA for $\alpha=0.5$
and  MHA for $\alpha=0.7$.
The model is trained with CIFAR100.}
\label{fig:cossim_cifar100_small}
\end{center}
\end{figure}

\begin{figure}[ht]
\begin{center}
\centerline{\includegraphics[width=0.95\columnwidth]{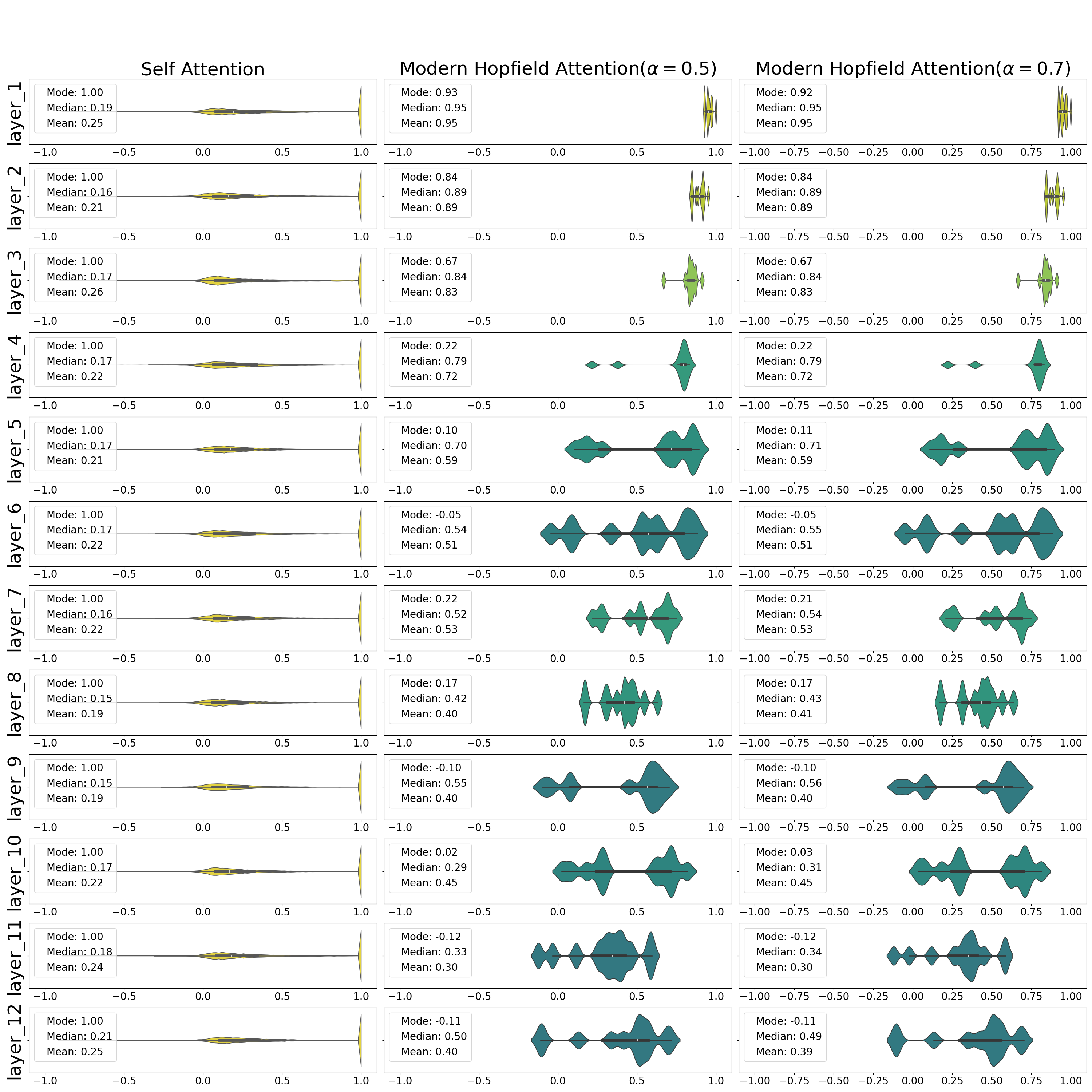}}
\caption{The violin plots of the cosine similarity of ViT-B with usual self-attention, MHA for $\alpha=0.5$
and  MHA for $\alpha=0.7$.
The model is trained with CIFAR100.}
\label{fig:cossim_cifar100_base}
\end{center}
\end{figure}

\begin{figure}[ht]
\begin{center}
\centerline{\includegraphics[width=0.75\columnwidth]{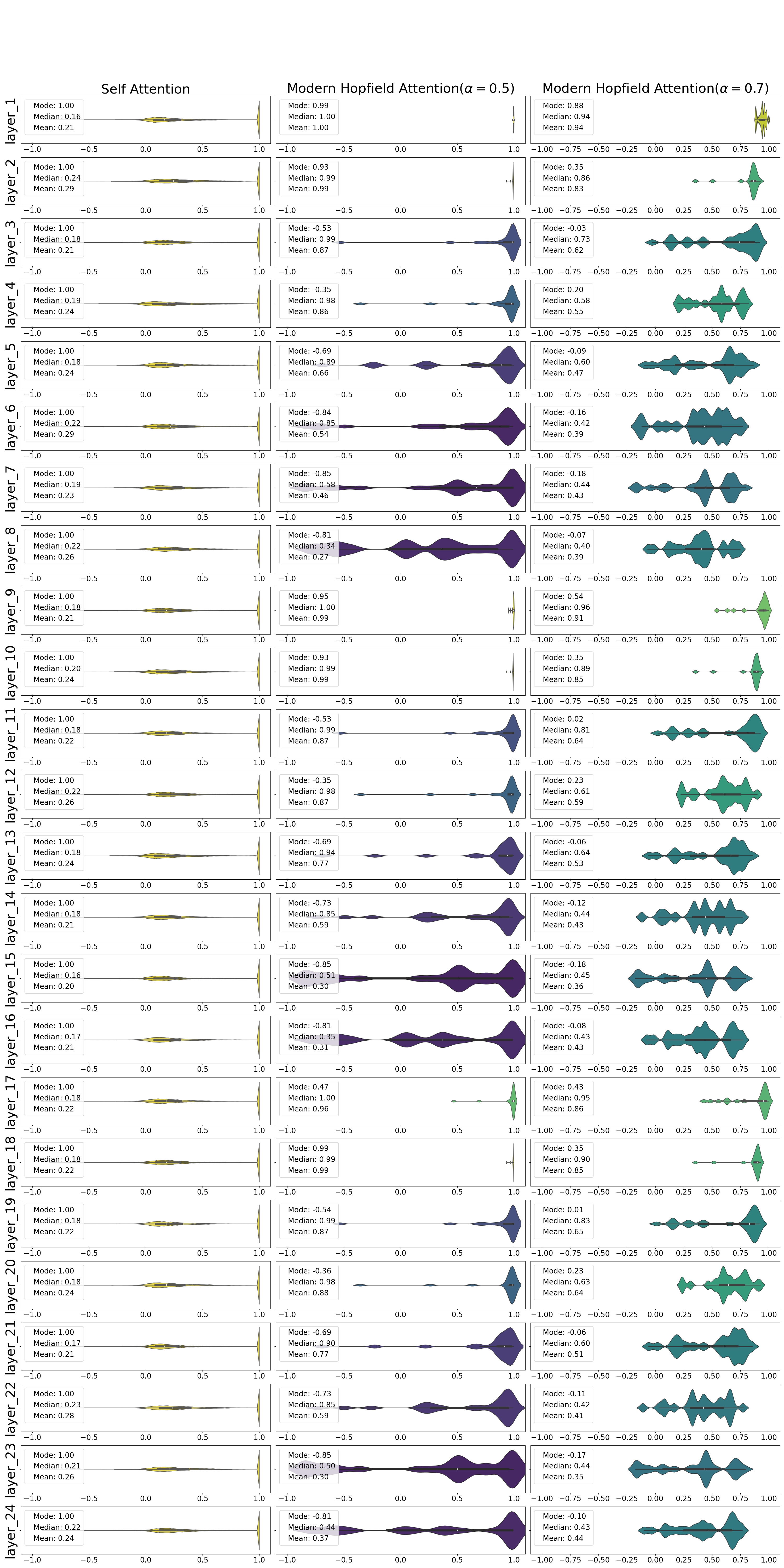}}
\caption{The violin plots of the cosine similarity of ViT-L with usual self-attention, MHA for $\alpha=0.5$
and  MHA for $\alpha=0.7$.
The model is trained with CIFAR100.}
\label{fig:cossim_cifar100_large}
\end{center}
\end{figure}

\section{Entropy Collapse}

Let's also look at other problems that deep Transformers have besides rank collapse.
Attention entropy collapse \cite{ghader2017does} is a measure of how much attention is focused on a small or large number of tokens, i.e., the degree of concentration of the attention distribution. A high attention entropy means that attention is directed to a large number of tokens, and the Transformer is distributing attention over a wide context. This is expected to allow each layer to construct a well contextualized embedding vector. On the other hand, when entropy is low, attention is directed to a small number of tokens.

When entropy is particularly low, it causes the problem of attention entropy collapse \cite{zhai2023stabilizing}, which is a source of instability in Transformer training. This problem tends to occur when hyperparameters are not carefully set. We also know that concentration of attentions can lead to overfitting to some undesirable vocabulary, which can greatly impair the generalization and fairness of the language model \cite{attanasio2022entropy,zayed2023should}.

Figure \ref{fig:cifar10_tiny_standard}-\ref{fig:cifar10_large_05}
compares the attention entropy of the vanila Transformers and our model. 
As can be seen from these figures, attention entropy tends to be small in some layers, but this property does not change even if attention weights are improved by MHA. In other words, MHA does not improve the performance of Transformer by improving entropy collapse.

Rank collapse means that diverse token vectors cannot be achieved, while entropy collapse means that uniform token mixing cannot be achieved. Therefore, this result is not surprising, as these two phenomena seem to be mutually exclusive. On the other hand, recent theoretical analysis \cite{bao2024self} shows that rank collapse and entropy collapse are compatible phenomena, and further research in these contexts is an interesting future direction.


\begin{figure}[h]
\begin{center}
\centerline{\includegraphics[width=0.95\columnwidth]{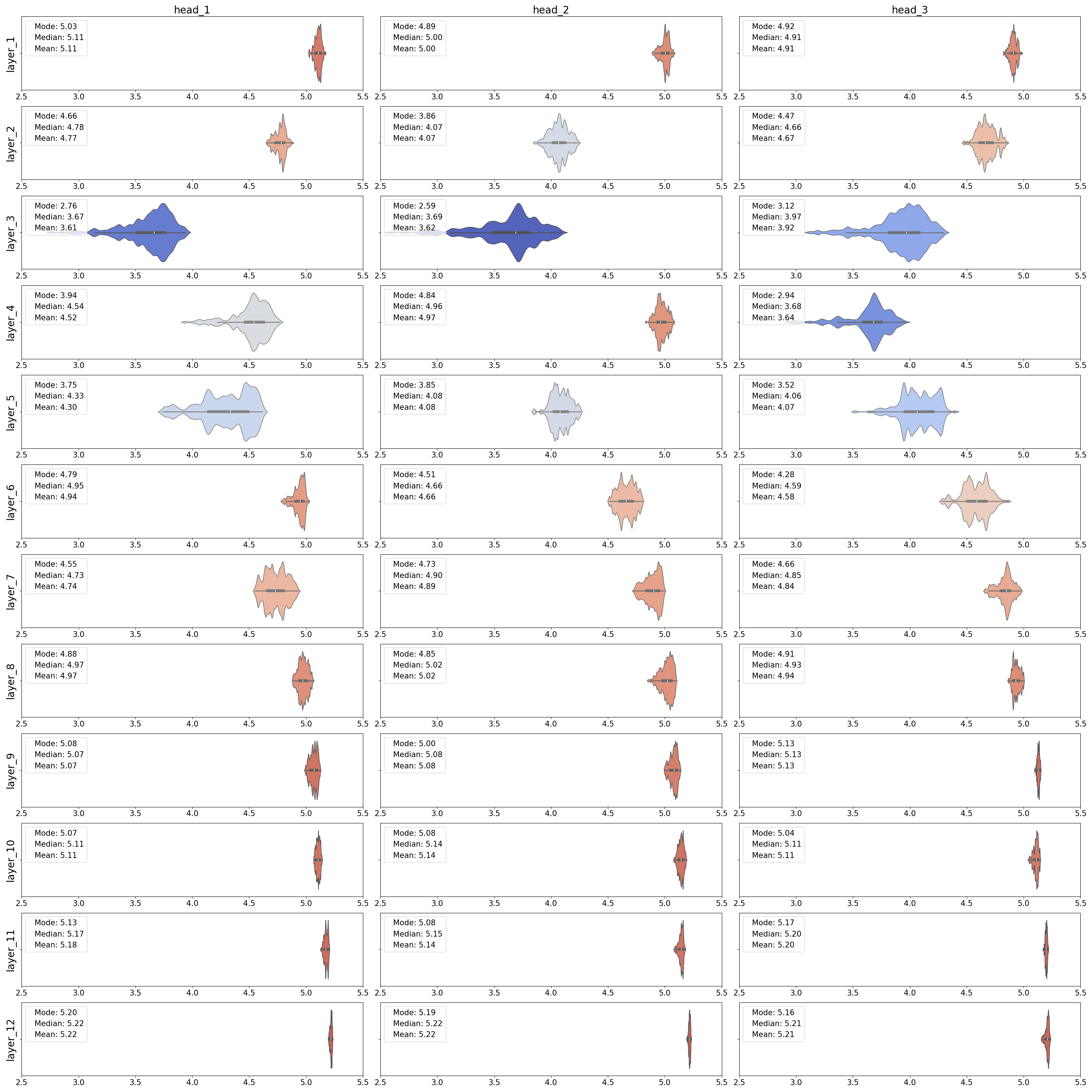}}
\caption{The violin plot of the attentional entropy for each layer and each head of ViT-T trained with CIFAR10 is shown.}
\label{fig:cifar10_tiny_standard}
\end{center}
\end{figure}

\begin{figure}[h]
\begin{center}
\centerline{\includegraphics[width=0.95\columnwidth]{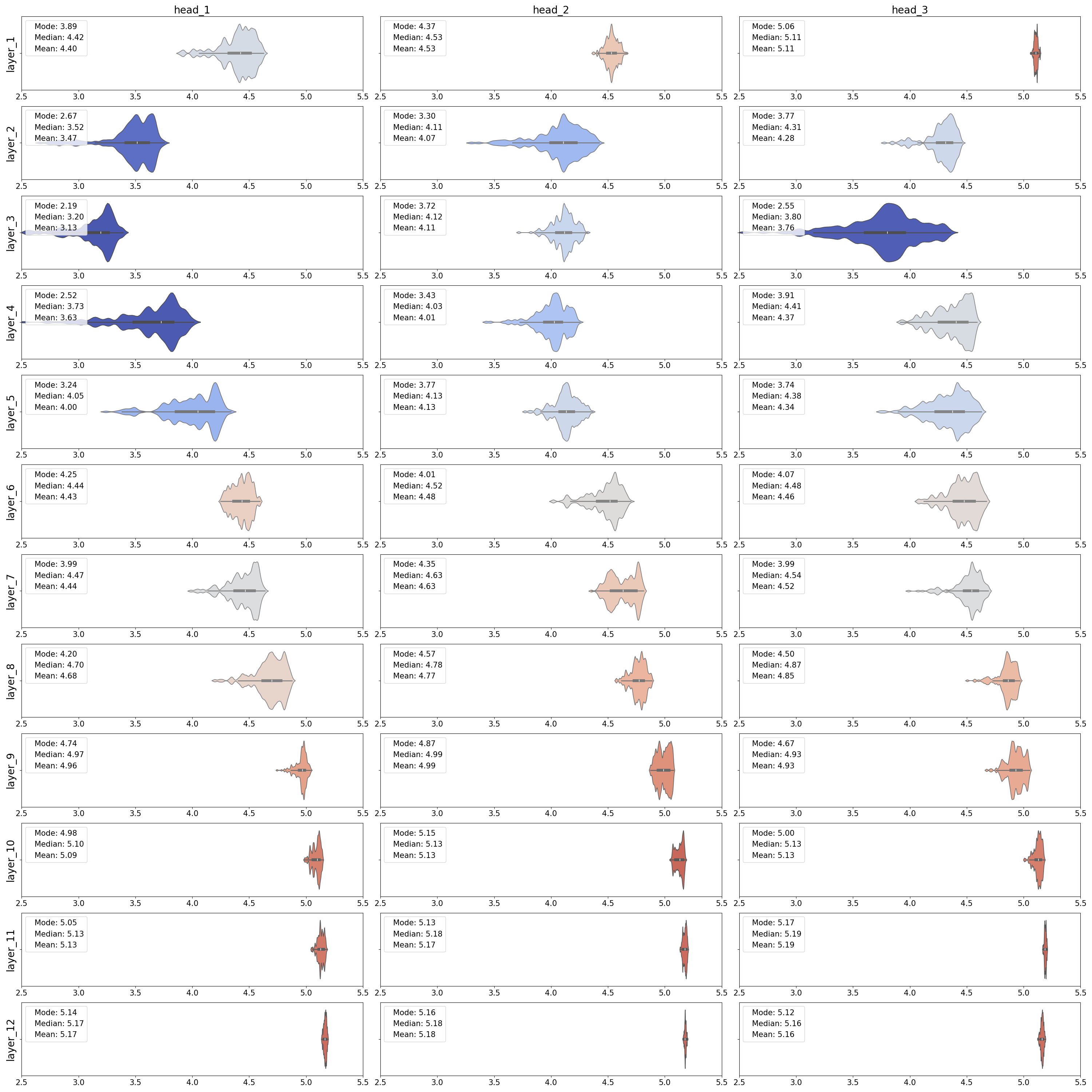}}
\caption{The violin plot of the attentional entropy for each layer and each head of MHA version of ViT-T ($\alpha=0.5$) trained with CIFAR10 is shown.}
\label{fig:cifar10_tiny_05}
\end{center}
\end{figure}



\begin{figure}[h]
\begin{center}
\centerline{\includegraphics[width=0.95\columnwidth]{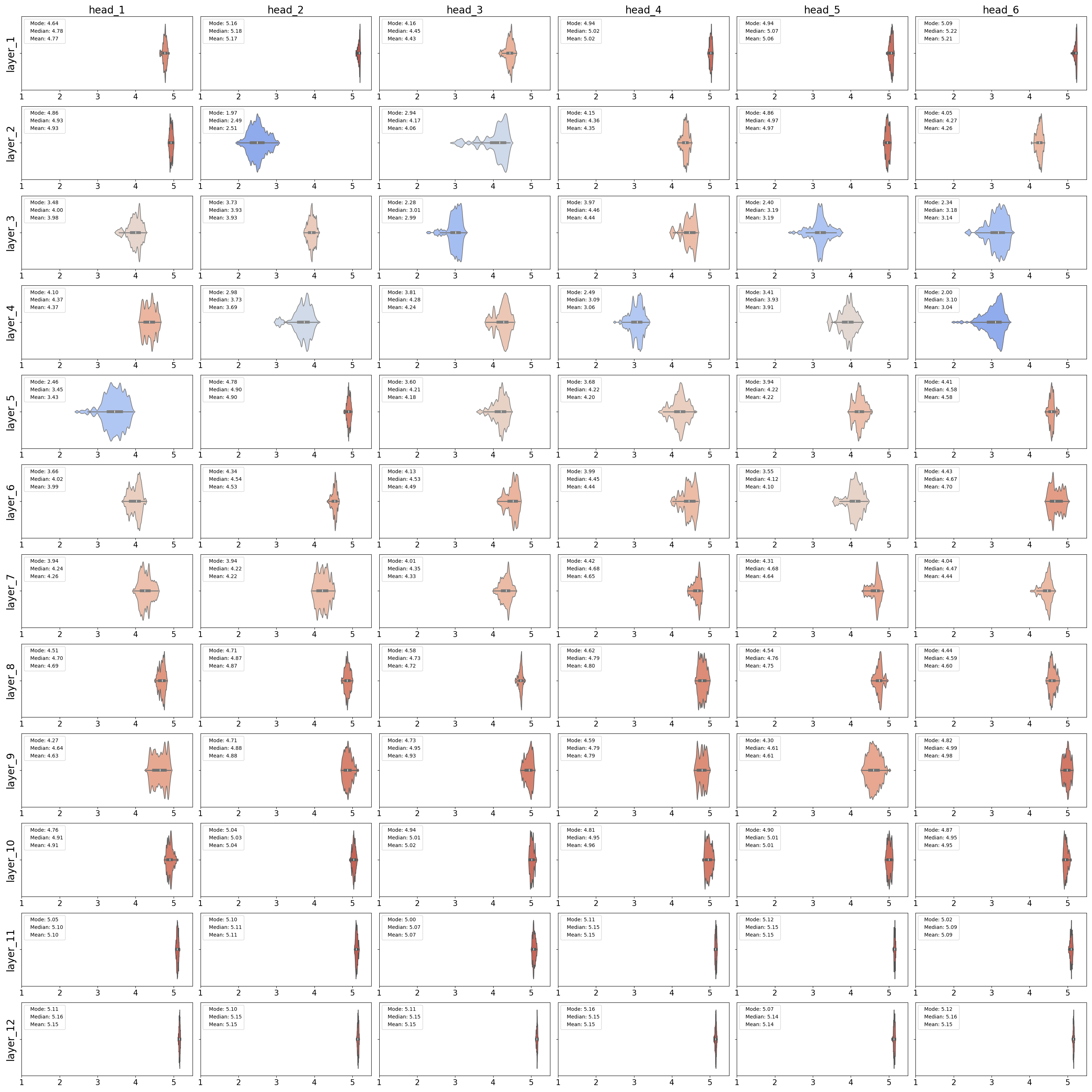}}
\caption{The violin plot of the attentional entropy for each layer and each head of of ViT-S trained with CIFAR10 is shown.}
\label{fig:cifar10_small_standard}
\end{center}
\end{figure}

\begin{figure}[h]
\begin{center}
\centerline{\includegraphics[width=0.95\columnwidth]{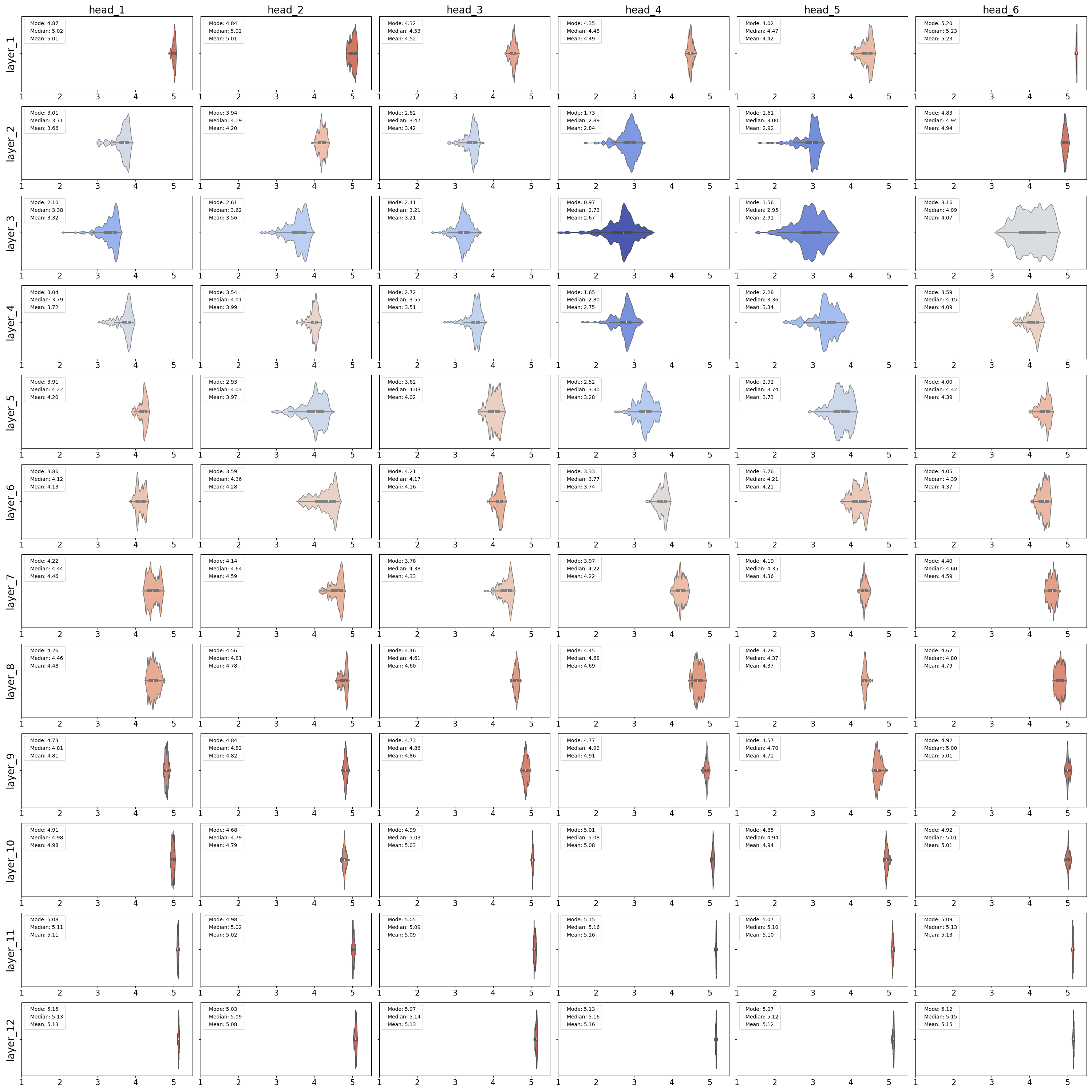}}
\caption{The violin plot of the attentional entropy for each layer and each head of MHA version of ViT-S ($\alpha=0.5$) trained with CIFAR10 is shown.}
\label{fig:cifar10_small_05}
\end{center}
\end{figure}


\begin{figure}[h]
\begin{center}
\centerline{\includegraphics[width=0.95\columnwidth]{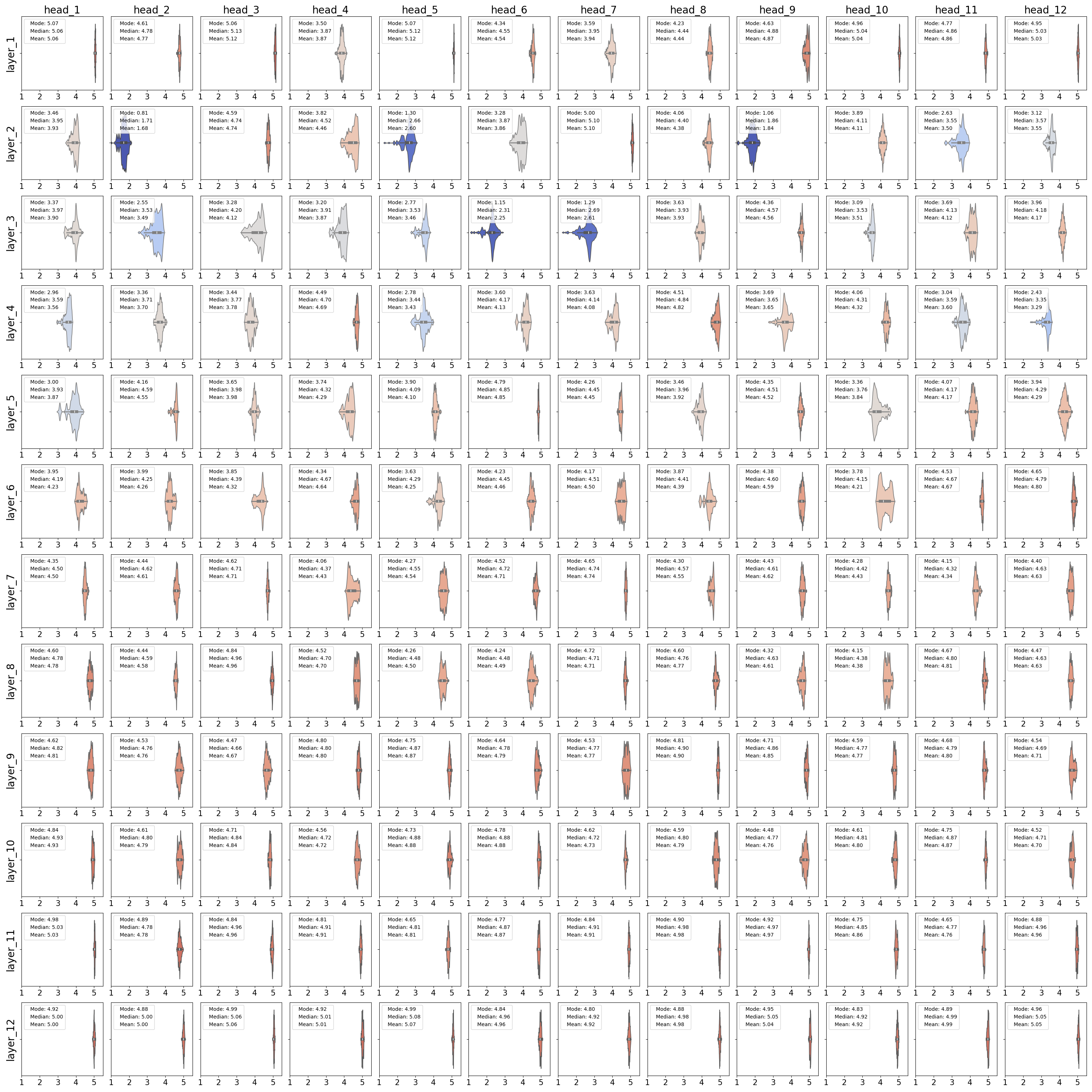}}
\caption{The violin plot of the attentional entropy for each layer and each head of ViT-B trained with CIFAR10 is shown.}
\label{fig:cifar10_base_standard}
\end{center}
\end{figure}

\begin{figure}[h]
\begin{center}
\centerline{\includegraphics[width=0.95\columnwidth]{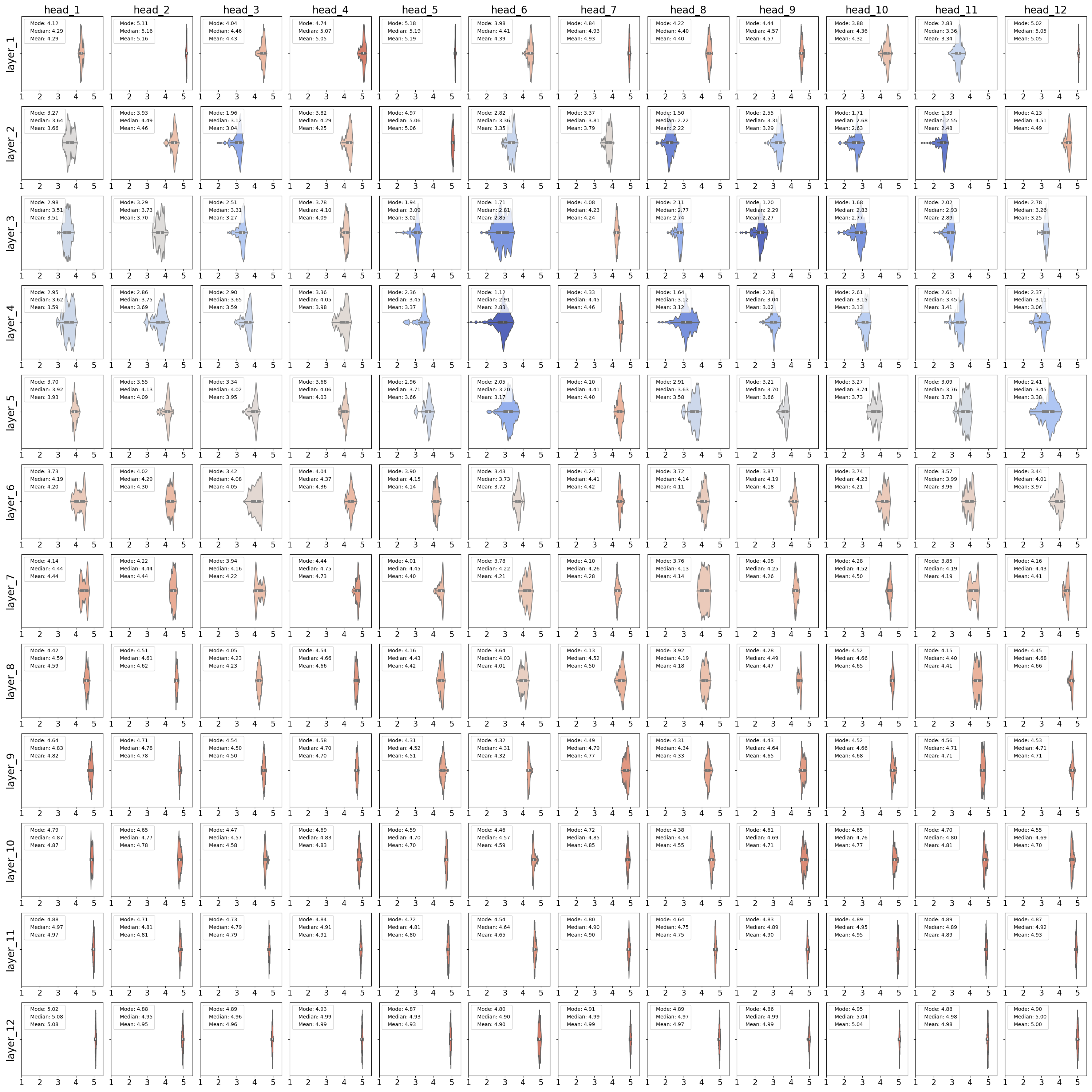}}
\caption{The violin plot of the attentional entropy for each layer and each head of MHA version of ViT-B ($\alpha=0.5$) trained with CIFAR10 is shown.}
\label{fig:cifar10_base_05}
\end{center}
\end{figure}


\begin{figure}[h]
\begin{center}
\centerline{\includegraphics[width=0.95\columnwidth]{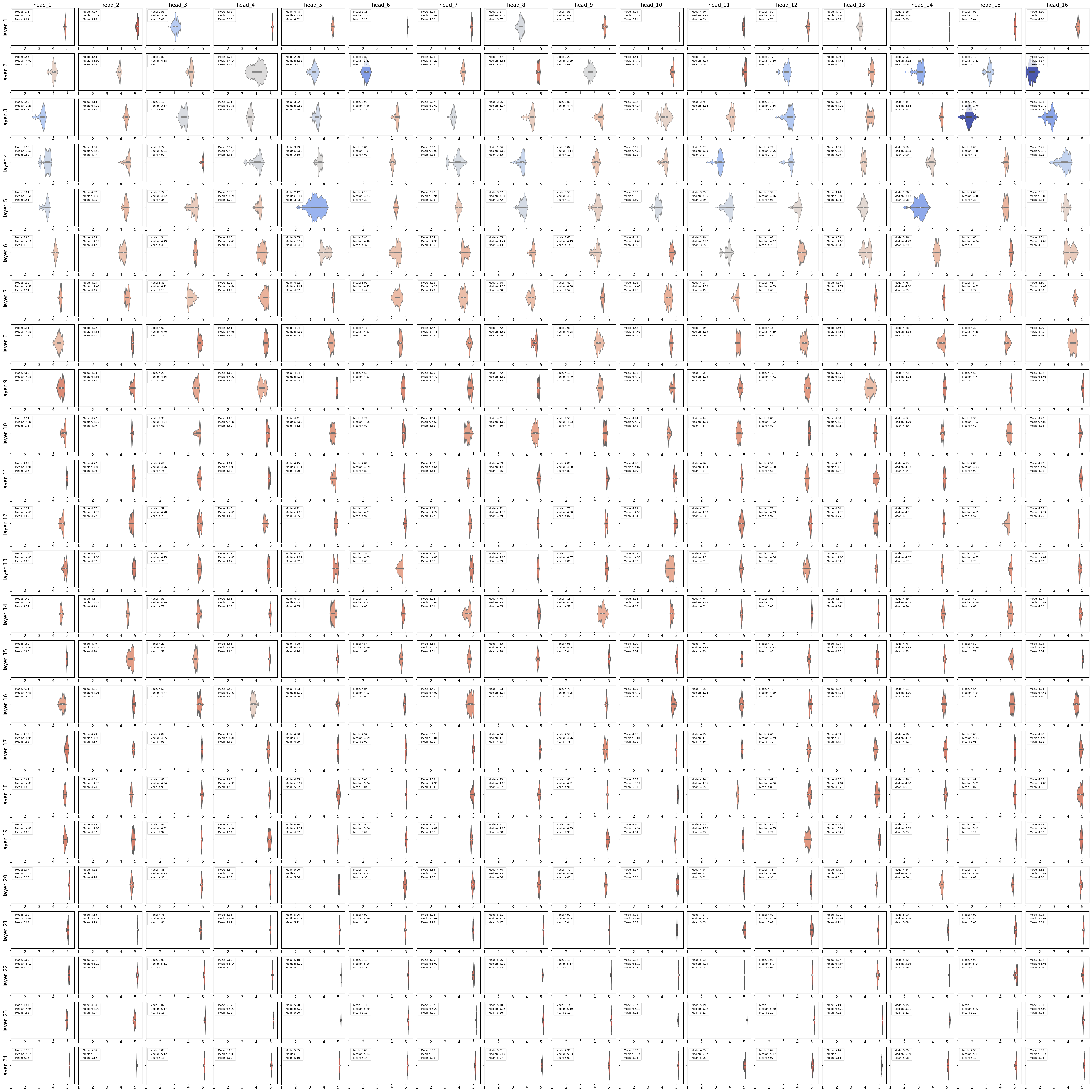}}
\caption{The violin plot of the attentional entropy for each layer and each head of ViT-L trained with CIFAR10 is shown.}
\label{fig:cifar10_large_standard}
\end{center}
\end{figure}

\begin{figure}[h]
\begin{center}
\centerline{\includegraphics[width=0.95\columnwidth]{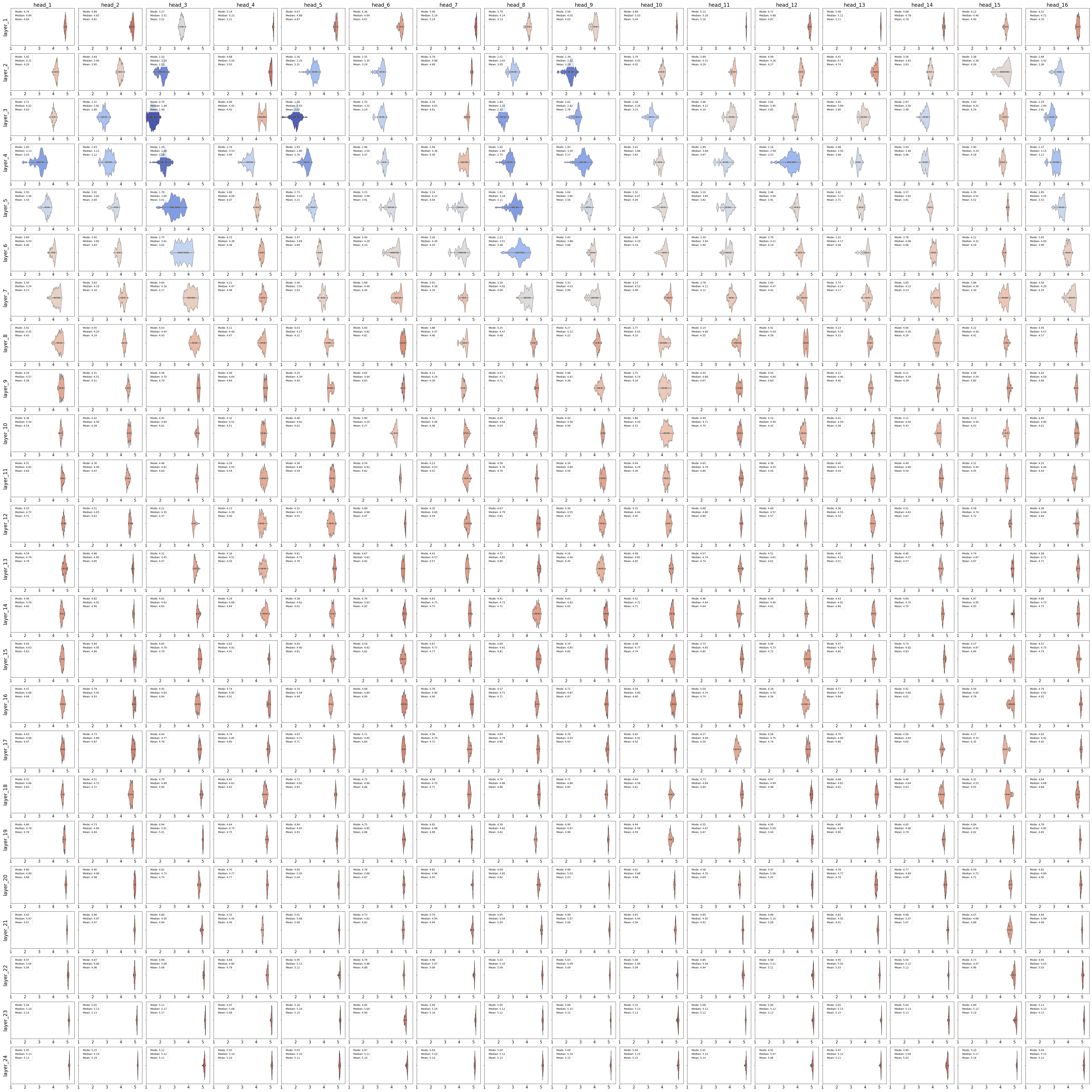}}
\caption{The violin plot of the attentional entropy for each layer and each head of MHA version of ViT-L ($\alpha=0.5$) trained with CIFAR10 is shown.}
\label{fig:cifar10_large_05}
\end{center}
\end{figure}

\end{document}